%% file: streaming_bandits.tex
\documentclass[11pt]{article}

\usepackage{amsthm}

\usepackage{graphicx} 
\usepackage{array} 

\usepackage{amsmath, amssymb, amsfonts, verbatim}
\usepackage{hyphenat,epsfig,subfigure,multirow}

\usepackage[usenames,dvipsnames]{xcolor}

\usepackage{algorithm}
\usepackage{algorithmic}

\usepackage{tcolorbox}
\tcbuselibrary{skins,breakable}
\tcbset{enhanced jigsaw}

\usepackage[compact]{titlesec}

\definecolor{DarkRed}{rgb}{0.5,0.1,0.1}
\definecolor{DarkBlue}{rgb}{0.1,0.1,0.5}

\usepackage{hyperref}
\hypersetup{
colorlinks=true,
pdfnewwindow=true,
citecolor=ForestGreen,
linkcolor=DarkRed,
filecolor=DarkRed,
urlcolor=DarkBlue
}

\usepackage{bm}
\usepackage{url}
\usepackage{xspace}
\usepackage[mathscr]{euscript}

\usepackage{mdframed}

\usepackage{cite}

\usepackage{enumitem}

\usepackage[margin=1in]{geometry}

\newtheorem*{claim*}{Claim}
\newtheorem*{proposition*}{Proposition}
\newtheorem*{lemma*}{Lemma}
\newtheorem*{problem*}{Problem}
\newtheorem*{remark*}{Remark}

\newtheorem{mdresult}{Result}

\newtheorem{mdinvariant}{Invariant}

\theoremstyle{definition}
\newtheorem{mdexample}{Example}[section]

\allowdisplaybreaks

\renewcommand{\qed}{\nobreak \ifvmode \relax \else
      \ifdim\lastskip<1.5em \hskip-\lastskip
      \hskip1.5em plus0em minus0.5em \fi \nobreak
      \vrule height0.75em width0.5em depth0.25em\fi}

\input{macros}

\title{A Sharp Memory-Regret Trade-Off for \\ Multi-Pass Streaming Bandits}
\author{Arpit Agarwal\thanks{Columbia University. E-mail: arpit.agarwal@columbia.edu} \and Sanjeev Khanna\thanks{University of Pennsylvania. E-mail: sanjeev@cis.upenn.edu}  \and Prathamesh Patil\thanks{University of Pennsylvania. E-mail: pprath@seas.upenn.edu}}

\begin{document}

\maketitle
\begin{abstract}
The stochastic $K$-armed bandit problem has been studied extensively due to its applications in various domains ranging from online advertising to clinical trials. In practice however, the number of arms can be very large resulting in large memory requirements for simultaneously processing them. In this paper we consider a streaming setting where the arms are presented in a stream and the algorithm uses limited memory to process these arms. Here, the goal is not only to minimize regret, but also to do so in minimal memory. Previous algorithms for this problem operate in one of the two settings: they either use $\Omega(\log \log T)$ passes over the stream \cite{rathod2021reducing, ChaudhuriKa20, Liau+18}, or just a single pass \cite{Maiti+21}. 

In this paper we study the trade-off between memory and regret when $B$ passes over the stream are allowed, for any $B \geq 1$, and establish \emph{tight} regret upper and lower bounds for any $B$-pass algorithm. Our results uncover a surprising \emph{sharp transition phenomenon}: $O(1)$ memory is sufficient to achieve $\widetilde\Theta\paren{T^{\half + \frac{1}{2^{B+2}-2}}}$ regret in $B$ passes, and increasing the memory to any quantity that is $o(K)$ has almost no impact on further reducing this regret, unless we use $\Omega(K)$ memory. Our main technical contribution is our lower bound which requires the use of \emph{information-theoretic techniques} as well as ideas from \emph{round elimination} to show that the \emph{residual problem} remains challenging over subsequent passes.

\end{abstract}

\section{Introduction}
The stochastic multi-armed bandit problem is a widely studied 
problem with applications in many domains such as online advertising, recommendation systems, clinical trials, financial portfolio design etc. 
In this problem, there are $K$ arms; in trial $t \in [T]$ the  algorithm pulls an arm $a_t$
and receives a reward drawn from
the reward distribution of $a_t$ with mean $\mu_{a_t}$.
The goal of the algorithm is to minimize the cumulative regret 
over $T$ trials where the regret for trial $t$ is defined as 
the gap between the maximum reward $\max_{i \in [K]} \mu_i$ and $\mu_{a_t}$.

In many practical applications such as online advertising and recommendation systems,
the number of arms can be very large and the learner might not be
able to store all the arms in memory. 
In these applications it can be more practical to process arms in a sequential
manner with small memory that is sub-linear in the number of arms.
Motivated by a long line of work on streaming algorithms in theoretical computer science \cite{Alon+99}, we consider a setting where 
the arms are presented in a (possibly adversarially chosen) stream and in each trial the 
algorithm decides whether to read the next arm 
from the stream into memory.
The algorithm can only store the indices and statistics of  $M$ arms out of the $K$ arms
and can only play an arm if it is present in the memory. 
The goal in this setting is to minimize the regret given a bounded amount of memory.

Previously, \cite{rathod2021reducing, ChaudhuriKa20, Liau+18} developed algorithms for regret minimization in this limited-memory streaming bandits setting, but their algorithms require a relatively large number of passes over the stream, with the former requiring $O(\log \log T)$ passes, and the latter two requiring $O(\log T)$ passes. 
Moreover, it is not understood whether the trade-off between memory and regret 
obtained by these algorithms is tight as a corresponding multi-pass lower bound is not known.
At the other extreme, \cite{Maiti+21} considered a $1$-pass streaming setting and showed that 
any algorithm using $M$ words (for  $M < K$) of memory needs to incur $\Omega(T^{2/3}/M^{7/3})$ expected regret. Also, there is a simple $1$-pass algorithm that uses $M=O(1)$ memory and achieves $O(K^{1/2} T^{2/3})$ regret.
These results of \cite{Maiti+21} imply that the $1$-pass setting exhibits
a sharp trade-off between memory and regret as explained below.

\emph{
The $1$-pass regret as a function of  memory $M$ has a sharp transition: 
with $M=O(1)$ one can achieve $O(T^{2/3})$ regret, and
increasing $M$ beyond $O(1)$ has 
little impact on further reducing this regret, unless we increase $M$ to $K$ in which case one can achieve $O(\sqrt{T})$ regret.\footnote{Note that when $M=K$ one can simply read all the arms in memory at once and use any stochastic multi-armed bandit algorithm such as the UCB algorithm to achieve a regret of $O(\sqrt{TK})$.}
}

In this paper we study a streaming setting for multi-armed bandits where we are allowed $B$ passes 
 over the stream, for any $B \geq 1$.
We seek to provide upper and lower bounds on the expected regret under a limited memory
budget and $B$ passes.
We also seek to understand the trade-off between memory and regret
as a function of the number of passes $B$.
Does increasing memory beyond $O(1)$ help in this $B$-pass setting or
is there again a sharp transition in regret similar to the $1$-pass setting?

Our main result is to prove a lower bound on the regret of any 
$B$-pass algorithm that uses limited amount of memory.
In particular, we show that any $B$-pass algorithm that uses $o(K/B^{2})$ words of memory necessarily incurs $\Omega\paren{4^{-B}T^{\half + \frac{1}{2^{B+2}-2}}}$ regret in expectation.
Note that for $B=1$ our result implies a tighter lower bound of $\Omega\paren{T^{2/3}}$ as compared to the $\Omega\paren{T^{2/3}/M^{7/3}}$ bound in \cite{Maiti+21}, for any 1-pass algorithm that uses $M < K/24$ words of memory.

Our lower bound exploits the 
main tension in the streaming setting: the algorithm has limited information about whether there are better arms further along in the stream, and hence, 
it is difficult to decide whether to keep exploring the current arms in memory or to 
 read more arms into memory by throwing away some of the current arms from memory. 
We construct a distribution over \emph{hard instances} such that, if 
in the first pass the algorithm performs \emph{sufficient exploration}
over (potentially `bad') arms then it already incurs a large regret in expectation.
If it performs \emph{insufficient exploration} in the first pass then 
it will throw away many `good' arms due to a limited memory budget
and will be unable to isolate the underlying instance at the end of the first pass. 
One of the main technical difficulty is to show that the resulting \emph{residual}  
distribution over instances is \emph{challenging} in a way that leads to large regret in the remaining $B-1$ passes. 
We overcome this difficulty by using information-theoretic techniques
to show that \emph{insufficient exploration} leads to low \emph{mutual information} 
which further leads to \emph{large entropy} in the residual distribution.
We then inductively argue that any high entropy 
distribution over instances will lead to large regret in the remaining $B-1$ passes.

We complement our lower bound with a simple $B$-pass algorithm
that uses $O(1)$ memory and achieves an expected regret upper 
bound of $\widetilde{O}\paren{T^{\half + \frac{1}{2^{B+2}-2}} \sqrt{KB}}$.
This implies that $O(\log \log T)$ passes and $O(1)$ memory 
are sufficient to achieve an almost optimal regret of $\Ot(\sqrt{KT})$,
and matches the recent $O(\log \log T)$ pass regret upper bound of \cite{rathod2021reducing}.
When $B=1$, we also recover the $O(T^{2/3} \sqrt{K})$ upper bound 
of \cite{Maiti+21}.
In short, our algorithm nicely interpolates the space between the $1$-pass $\Ot(T^{2/3} \sqrt{K})$ regret and the $(\log\log T)$-pass $\Ot(\sqrt{KT})$ regret upper bounds as function of the number of passes $B$.

Our algorithm is based on two key operations: (i) estimating the reward of the best arm,
(ii) identifying sub-optimal arms based on this estimate.  
In each pass the algorithm sets a maximum budget for the number of pulls allowed for each arm, and this budget keeps increasing over successive passes. 
The algorithm reads an arm into memory and pulls this arm until it is identified as a sub-optimal arm or the maximum budget is exceeded.
The estimate of the maximum reward is then updated and the next arm is read into memory. Since the budget keeps increasing over passes, the estimate 
for the maximum reward becomes more refined, and
sub-optimal arms are identified more easily.

Our lower and upper bound together imply the following (perhaps surprising) sharp threshold phenomenon in our $B$-pass setting.

\emph{
The $B$-pass regret as a function of  memory $M$ has a sharp transition: 
with $M=O(1)$ one can achieve $\widetilde\Theta\paren{T^{\half + \frac{1}{2^{B+2}-2}}}$ regret, and
increasing $M$ to any quantity that is $o(K/B^2)$ has 
almost no impact on further reducing this regret.
}

\noindent
\textbf{Related Work.}
The stochastic multi-armed bandit problem has been extensively studied in many fields including operations research, statistics and machine learning. 
We refer the reader to excellent surveys 
in \cite{Bubeck+13, Slivkins19},
and only mention work that is directly relevant to our streaming setting. 
\cite{Liau+18} studied a limited memory setting 
for multi-armed bandits and showed that one 
can achieve (almost) instance-wise optimal regret in $O(\log T)$
passes and $O(1)$ memory.
\cite{ChaudhuriKa20} studied a similar setting 
and showed that with $O(\log T)$ passes 
and $M$ memory one can achieve a regret upper bound
of $\Ot(KM+ \frac{K^{3/2}}{M} \sqrt{T})$.
However, these works only considered a $O(\log T)$-pass setting and did not study the trade-off between  
memory and regret for any arbitrary number of passes $1\leq B < \log T$.
A recent work \cite{rathod2021reducing} achieves a regret upper bound of $\Ot(\sqrt{KT})$ in $O(\log \log T)$ passes. However, their work does not address the question of the regret achievable (both upper and lower bounds) for any arbitrary number of passes $1\leq B< \log\log T$.
As discusses earlier, 
\cite{Maiti+21} considered a  $1$-pass 
streaming setting but their results do not apply for $B>1$ passes,
which is the main focus of our paper.
There is also some work on best arm identification 
with limited memory in the streaming setting. \cite{AssadiWa20} show that 
one can identify the best arm with $1$ pass over the stream and $O(1)$ memory using  
$O(K/\Delta^2)$ sample complexity where $\Delta$ is the minimum gap between the best arm and any other arm. 
\cite{Jin+2021} further obtain instance-wise 
optimal sample complexity for this problem using $\log 1/\Delta$ passes and $O(1)$ memory.

The stochastic multi-armed bandits problem has also been studied under the setting of limited adaptivity \cite{gao2019batched, Perchet+15}.
Under this setting, an algorithm operates in rounds and in each 
round it plays arms according to a fixed distribution that can only depend on the outcomes from the previous rounds. 
Even though the tradeoff between rounds and regret in this setting is similar to the tradeoff between 
passes and regret given limited memory in our setting, the key difference between the two settings is that this setting necessarily requires at least 1 bit of information per arm for a total of $\Omega(K)$ memory, but cannot be adaptive within a batch, whereas in our setting, we can be fully adaptive within a pass but are given strictly less than $K$ memory. 
Due to this difference the challenges in these two settings are quite different,
which is reflected in the fundamentally different techniques used in the respective lower bounds.

Very recently, independently of our work, \cite{Srinivas+22} studied the problem of
online learning with expert advice in a streaming setting and 
established a trade-off between regret and memory in this setting.
However, there are several fundamental differences between the multi-armed bandits
problem studied here and the experts problem studied in \cite{Srinivas+22}--
(1) in the experts problem one gets to see the loss of every expert at every trial, whereas in our problem one only gets to see the reward of the arm that is played,
(2) in \cite{Srinivas+22} the losses on experts are generated adversarially whereas in our work the rewards of arms are generated stochastically, 
(3) in \cite{Srinivas+22} the stream consists of the prediction of experts for each trial,
whereas in our work the stream consists of the arms.
As a result, the two settings require very different techniques for proving lower 
and upper bounds, and neither result has any implications on the other.

\noindent
\textbf{Organization.}
In \Sec{sec:problem} we discuss the problem setting and set up relevant notation.
We discuss our lower bound on regret in \Sec{sec:lower} which is the main result of our paper.
We then provide an upper bound on regret in \Sec{sec:upper},
and finally conclude in \Sec{sec:conc}.

\section{Problem Setting}
\label{sec:problem}
We study a multi-armed bandit problem, where the instance consists of a finite set $\K$  of  ($K = |\K|$) arms and a time horizon $T$ which is known ahead of time. When any arm $a\in \K$ is played, an i.i.d. reward is drawn from its corresponding reward distribution defined over $[0,1]$ with mean $\mu_{a}$ of which the algorithm has no prior knowledge.\footnote{We assume that the support of the reward distributions is $[0,1]$ for ease of analysis; our algorithmic results can be easily extended to sub-Gaussian distributions over arbitrary support.} The objective in this setting is to minimize the cumulative regret, which is defined as $R_T := \sum_{t=1}^T (\max_{j \in \K} \mu_j - \mu_{a_t})$ where $a_t$ is the arm played in trial $t\in [T]$.

We assume a limited memory setting where the arms $\K$ are presented to the algorithm as an \emph{arbitrarily (or adversarially) ordered read-only stream}
, and the algorithm is restricted to store the \emph{identities} and the corresponding \emph{statistics} of at most $M<K$ arms simultaneously while being allowed at most $B\geq 1$ passes over the stream. The input parameters $T$, $K$, $B$ and $M$ are assumed to be stored for free ($O(1)$ space). Crucially, \emph{the algorithm can only play an arm if it is in its memory}. Therefore, in each trial $t\in [T]$, the algorithm must decide to either play an arm currently present in its memory, which generates a reward (potentially incurring regret) and consumes a trial, or read the next arm from the stream into memory, which neither incurs regret nor consumes a trial. If the algorithm chooses to do the latter and the memory is full, then it must first discard some arm to accommodate the new arm, in which case both the statistics as well as the identity of the discarded arm are forgotten. Furthermore, the discarded arm cannot be read into memory (and hence played) until it is encountered again in a future pass over the stream.

\begin{rem}
In the above setting the set of arms in the stream remains the same even though their order can change adversarially. 
One can also consider a modified setting where the algorithm 
is allowed to delete arms from the stream permanently so that they do not appear
in future passes. For example, it might want to delete these arms if it identifies that these arms are suboptimal
and do not need to be processed further. 
We note that both our lower bound and upper bound results apply to this modified setting. 
\end{rem}

\textbf{Notation.} In the rest of this paper, we use upper case letters to refer to instance dependent constants, such as the length of the time horizon $T$, number of arms $K$, number of passes $B$, and the memory size $M$. We use $\B$, $\dist$, $\psi$ and $\phi$ to refer to distributions, and $\event$ to refer to events. We use other upper case calligraphic letters to refer to sets, and other lower case English or Greek letters to refer to miscellaneous constants. Lastly, we use $\log$ base $2$, and $\ln$ for natural logarithms.

We denote random variables in sans serif font, e.g., $\rv{X}$.
For a random variable $\rv{X}$, $\supp{\rv{X}}$ denotes the support of $\rv{X}$ and $\distribution{X}$ denotes its distribution.
We denote the \emph{Shannon Entropy} of a random variable $\rA$ by
$\en{\rA}$ and the \emph{mutual information} of two random variables $\rA$ and $\rB$ by
$\mi{\rA}{\rB} = \en{\rA} - \en{\rA \mid \rB} = \en{\rB} - \en{\rB \mid \rA}$. 
A summary of useful information theory facts is given in \App{sec:facts}.

\section{A Regret Lower Bound for Limited Memory Multi-Pass Algorithms}
\label{sec:lower}

Our main result, which is an information-theoretic lower bound on the cumulative regret that can be achieved 
by any $B$-pass algorithm with limited memory, is presented in the following theorem.

\begin{theorem}
\label{thm:multi_pass_lower}
Given a time horizon $T$, a stream of $K$ arms, and passes $1\leq B < \log\log T$ over this stream, there exists a distribution over $K$-armed bandit instances such that any $B$-pass algorithm that 
uses at most $K\cdot(8B(B+1)\log e)^{-1}$ memory suffers 
$\Omega\paren{4^{-B}T^{2^B/(2^{B+1} - 1)}}$ regret in expectation.
\end{theorem}

This lower bound paints a rather pessimistic picture for regret minimization in a limited memory streaming setting. Given any constant number of passes, we need $\Omega(K)$ memory to achieve $O(\sqrt{T})$ regret that is already achievable by a single pass algorithm with memory $K$. 
Furthermore, for any given memory $M$ up to $ o(K/\log^2\log T)$, a superconstant $\Omega(\log\log T)$ number of passes are required to achieve this optimal regret. In \Sec{sec:upper}, we will show another surprising result on the threshold nature of memory: for a fixed number of passes $B$, the regret achieved by a constant memory algorithm is asymptotically no different from that achieved by any $o(K/B^2)$ memory algorithm in the worst case. In other words, for any fixed value of $B$,
the best achievable regret does not go down as we increase memory unless we are ready to allow
$\Theta(K/B^2)$ memory.

To the best of our knowledge, this is the first regret lower bound 
for any $B>1$ number of passes. Moreover, for $B=1$, we achieve 
a lower bound of $\Omega(T^{2/3})$ for any $M < K/(16\log e)$ which 
improves upon the 
$\Omega(T^{2/3}/M^{7/4})$ lower bound of \cite{Maiti+21} under the same setting of $B=1$.
We now begin proving our lower bound.

\subsection{Proof of \Thm{thm:multi_pass_lower}}
\label{sec:lb-overview}

At a high level, our lower bound exploits the fact that any limited memory algorithm must operate conservatively due to the presence of arms for which it has absolutely no information until they are actually encountered in the stream. Since only a limited number of arms can be explored at any given time, any limited memory algorithm   faces the following dilemma. (1) Spend enough time playing the arms it has in memory and gain some meaningful information about them, but then potentially run the risk of acquiring large regret in the event there is some high value arm yet to be seen, or (2) Try to quickly move ahead in the stream, discarding arms in memory after a few samples, but then potentially risk throwing away good arms due to lack of sufficient information. Since the decision to throw away arms is irrevocable, and the statistics and identities of the discarded arms are forgotten, the algorithm would then have one fewer pass to rectify its mistake in the event that no obviously high value arms are found ahead in the stream.

In order to prove our lower bound, without loss of generality, we will assume that the stream order does not change across passes\footnote{The regret guarantees of our algorithm (\Sec{sec:upper}) hold even when the stream order changes adversarially.}. 
We will restrict our attention to deterministic algorithms 
in the proof. This is because a lower bound 
for deterministic algorithms on a suitable distribution over instances
also implies a lower bound
for randomized algorithms
using Yao's minimax principle \cite{yao1977probabilistic}. 
Formally, if there is any randomized algorithm with \emph{low expected 
regret}, then there exists a choice of random bits  
such that this algorithm has low expected regret given this choice of bits.
Therefore, by conditioning on these random bits, we have a deterministic algorithm 
that has low expected regret which will contradict the lower bound.

Our lower bound is based on the general idea of `round elimination'
used for proving communication
complexity lower bounds where one inductively argues that 
the residual instance at the end of the each round 
will remain `hard' over subsequent rounds. 
Our $B$-pass lower bound constructs `hard' instances over $K$ arms by
composing together $B+1$ layers of `hard' instances over subsets of arms.
We partition the stream of $K$ arms into contiguous subsets 
of size $K/(B+1)$ and the $j$-th layer of hard instances is defined over 
the $j$-th $K/(B+1)$-sized subset.
At a high level, we argue
that after performing $j$
passes an algorithm will either incur `large regret'
or will only be able to `peel-off' the last $j$
 layers.
In other words, if the algorithm has 
not incurred `high regret' at the end of 
$j$ passes, then it still needs to solve  
a hard problem over at least 
$B+1-j$ layers.

Within each layer $j\in [B+1]$, we generate a `hard' instance by sampling a special arm $i^*_j$ from a \emph{near-uniform} distribution over the arms in that layer, whose mean reward is \emph{nearly-equally-likely} to be either \emph{low}, namely $\mu_{i^*_j} = 1/2$, or \emph{high}, namely $\mu_{i^*_j} = 1/2+\Delta_j$ where $\Delta_j$ is a parameter that we will specify shortly. All other arms in the layer have mean reward $1/2$. This potential ``high'' reward of $1/2+\Delta_j$ increases across layers, with $\Delta_1$ being the smallest, and $\Delta_{B+1}$ being the largest. This intuitively forces any algorithm to rush through all of the initial $B$ layers, because the regret would be massive if $i^*_{B+1}$ realizes to have a high reward, the odds of which are nearly half. However, in doing so the algorithm
will learn very little about the special arms in first $B$ layers,
and will have to solve a hard problem over these layers in the 
remaining $B-1$ rounds.

In order to formalize the above construction,
we define a distribution over `hard' instances for a single layer that
 is parameterized by the set of arms $\A$ in the layer, the mean reward parameter $\Delta$ for the best arm in that layer, and a nearly-uniform joint distribution $\psi$ over $\A\times\{0,1\}$ for sampling the best arm and its reward.

\textbox{Distribution $\distD_{\A}$: \textnormal{Given a set of arms $\A$, a joint distribution $\psi$ over support $\A\times \{0,1\}$, and parameter $\Delta \leq 1/4$}}{
\begin{itemize}
	\item 
	Sample $(i^*, y) \sim \psi$ such that $i^* \in \A$ and $y \in \{0,1\}$. For all $i\in \A$, let
	\begin{align*}
    		\mu_{i} = \begin{cases}\half + y \Delta \,, \text{ if } i= i^* \\
		\half \,\qquad \quad\text{ otherwise }
		\end{cases}                	
	\end{align*}
	\item Return the arms $\A$ with Bernoulli reward distributions with means $\{\mu_i\}_{i\in \A}$. 
\end{itemize}
\label{text:dist1}
}

Note that the special arm $i^*$ in layer is also a best arm within the 
layer.
We will now define what is means for the distribution $\psi$ 
of the special arm $i^*$ to be $\gamma$-nearly uniform.

\begin{definition}[$\gamma$-nearly uniform $\psi$]
\label{defn:nearly_uniform}
Given a set of arms $\A$, a joint distribution $\psi$ over support $\A\times \{0,1\}$,
and  $\gamma > 0$,
we say that $\psi$ is $\gamma$-nearly uniform if the random variables 
 $(\bI,\bY) \sim \psi$ are such that 
 $\en{\bI} \geq \log A - \gamma$
and $\en{\bY|\bI} \geq \log 2 - \gamma$.
\end{definition}

This following key lemma quantifies how little any algorithm would actually learn about the special arm in a layer if this arm is $\gamma$-nearly uniformly distributed and the algorithm rushes through this layer, i.e. collects very few samples. 

\begin{lemma}
\label{lem:prob_en}
Given a time-horizon $T$, a set of arms $\A$ of size $A = |\A|$,
with mean rewards generated according to a distribution $\distD_\A$
where $\psi$ is $\gamma$-nearly uniform for some $\gamma \geq 0$.
Let $(\bI,\bY) \sim \psi$ and let $\ALG$ be any deterministic algorithm that \emph{adaptively} pulls arms in $\A$. 
Let $\sigma \in [T]$ be the \emph{randomly chosen} stopping time 
of the algorithm and $\S_\sigma = (j_t, r_t)_{t=1}^\sigma$ be execution history of $\ALG$ with $j_t$ being the arm pulled and $r_t$ being its observed reward in trial $t$, respectively.
For a given input parameter $\beta < 1$, let $\M \subset \A$ be any set of size $\beta A$ chosen to be retained in memory by $\ALG$ after observing the execution history $\S_\sigma$.  If $\Ex[\sigma] \leq \frac{\epsilon^2}{6 \Delta^2}$ for some $\epsilon >0$, then the event
\begin{align*}
\event = \Big( \bI \notin \M \text{ and } \en{\bI \mid \S_\sigma, \bI \notin \M} \geq \log ((1-\beta)A) - \frac{\gamma+\epsilon}{1-\log(1+\beta)-\gamma-\epsilon} \\
		\text{ and } \en{\bY\mid \S_\sigma,\bI,\bI \notin \M} \geq \log 2 - \frac{\gamma + \epsilon}{1-\log(1+\beta)-\gamma - \epsilon} \Big)
	\,,
\end{align*}
occurs with probability at least $1 - \log(1+\beta)-\gamma -3\epsilon$ over the samples seen by the algorithm.
\end{lemma}

A formal proof of this lemma is given in \Sec{sec:prob_en}.
One can interpret this lemma as follows-- given a set of instances where the 
best arm and its reward 
 are sampled from a $\gamma$-nearly uniform distribution, then no algorithm can hope to trap the best arm in a small subset (its memory) after a period of insufficient exploration with any considerable probability. Moreover, in the event that the best arm is discarded, nothing meaningful is learned either as the best arm is nearly-equally likely to be any of the discarded arms, and its reward is nearly-equally likely to be either low or high. Thus, the entropy of the identity of the best arm as well as its reward remains large in the posterior distribution over the discarded arms induced by the samples observed by the algorithm. This observation will be important to show that in the event $i^*_{B+1}$ realizes to have a low reward, the algorithm still faces a hard distribution consisting of $B$ layers while having depleted one of its passes. We will now ``stitch'' together these $(B+1)$ layer-wise hard distributions into a hard distribution over all $K$ arms.

\textbox{Distribution $\dist_{\K, B}^{\{\psi_j\}_{j=1}^{b+1}}$:  
\textnormal{Given a set of arms $\K$ of size $K=|\K|$, an integer $B \in \IN_+$, and a set of $(b+1)\leq (B+1)$ joint distributions $\{\psi_j\}_{j=1}^{b+1}$ 
where each $\psi_j$ is supported over $\A_j\times \{0,1\}$ with $\{\A_j\}_{j=1}^{b+1}$ being a contiguous and sequential partition of $\K$ into sets of equal size $K/(b+1)$.
}}{
\begin{itemize}
	\item For $j \in [b+1]$, define 
		\begin{align*}
		\Delta_j = \frac{T^{-\frac{2^{B} - 2^{j-1}}{2^{B+1} - 1}}}{4}
		 \,.
		 \end{align*}		 
	\item For $j \in [b+1]$, sample mean reward parameters $\{\mu_i\}_{i \in \A_j}$ according to $\dist^{\Delta_j,\psi_j}_{\A_j}$.
	\item Define 
		\begin{align*}
			\dist_{\K,B}^{\{\psi^j\}_{j=1}^{b+1}} := \dist^{\Delta_1,\psi_1}_{\A_1} \otimes  \dist^{\Delta_2,\psi_2}_{\A_2} \otimes  \cdots \otimes  \dist^{\Delta_{b+1},\psi_{b+1}}_{\A_{b+1}}  
		\end{align*}
	\item Return the arms $\K$ with the reward distribution of arm $i\in\K$ being Bernoulli $\B(\mu_i)$.  
\end{itemize}
}

In the above distribution, one should think of
$B$ as an input parameter that 
corresponds to the number of passes allowed 
to the algorithm at the start, whereas 
$b$ as the number of passes that are remaining at an 
intermediate step in the algorithm. 
We need to define our distribution for any $b \leq B$
as we need to show that residual distributions
remain `hard' at every intermediate step in the algorithm.
Hence, if there are $b$ passes remaining, there will be a `hard' $(b+1)$-layered residual problem still to be solved by the algorithm.

Armed with this hard distribution, we are now ready to prove the lower bound as follows. 
Let there be $b$ passes remaining at an intermediate step in the algorithm,
and let the distribution of rewards be according to $\dist_{\K, B}^{\{\psi_j\}_{j=1}^{b+1}}$
such that the special arm in each of the $b+1$ layers is nearly uniformly distributed. 
The algorithm is presented with each layer one by one in the stream.
We divide the execution of the algorithm into $b+1$ epochs
where the $j$-th epoch begins when the first arm in layer $j$ is read into memory and ends right before the first arm from layer $j+1$ is read into memory.

Let $\alpha = 2^B/(2^{B+1}-1)$, and let the available memory be $\beta K/(b+1)$ for an 
appropriately chosen $\beta \in (0,1]$. Since the number of arms in each layer is $K/(b+1)$, the algorithm needs to discard at least $(1-\beta)$ fraction of the arms from each layer. Now suppose for any of the first $j\in [b]$ epochs, the algorithm actually collects at least $\epsilon^2/\Delta_j^2$ (for some small $\epsilon$) samples in that epoch, then we are already done, as the algorithm will suffer $\Omega(\epsilon^2 \Delta_{b+1}/\Delta_j^2) = \Omega(T^\alpha)$ regret if the reward of $i^*_{b+1}$ realizes to its high value, the odds of which are nearly half. On the other hand, if the algorithm does not explore enough in every epoch, then for sufficiently small $\beta,\epsilon$, the event described in \Lem{lem:prob_en} will occur for all of the initial $b$ epochs with constant probability. As a result, the posterior reward distributions over the $(1-\beta)$ fraction of the rejected arms from every layer will provably remain hard (as per our definition of a hard instance for an epoch). Therefore, if the reward of $i^*_{b+1}$ realizes to its low value, the odds of which are nearly half, the algorithm now faces a hard distribution with $b$ layers over $b-1$ passes, at which point we will appeal to induction to show that the regret of this algorithm in this case must also be large. This idea is formalized in the following lemma.

\begin{lemma}
\label{lem:multi_pass_lower}
Let $K,T,B,b \in \IN_+$ be any set of parameters such that $K\leq T$, and $1\leq b \leq B <\log\log T$. 
Let $\{\A_j\}_{j=1}^{b+1}$ be a contiguous partition of arms $\K$ such that for each $j\in [b+1]$, $|\A_j| = A = K/(b+1)$. Furthermore, 
let $\{\psi_j\}_{j = 1}^{b+1}$ be any set of distributions such that 
each $\psi_j$ is $\gamma$-nearly uniform (see \Def{defn:nearly_uniform}) for $0 \leq \gamma \leq 1/(32b)$.
Given a stream of arms $\K$ with mean rewards sampled according to $\dist_{\K, B}^{\{\psi_j\}_{j=1}^{b+1}}$, the expected regret $R_T$ of any $b$-pass deterministic algorithm that uses at most $M=K(8b(b+1)\log e)^{-1}$ words of memory is bounded as
\[\Exp[R_T]\geq\Omega\paren{{4^{-b}T^{\frac{2^{B}}{2^{B+1} - 1}}}}.\]
\end{lemma}

A formal proof of this lemma is provided in \Sec{subsec:multi_pass_lower}.
The proof of our main result in \Thm{thm:multi_pass_lower}
now follows easily from the above lemma by setting $b = B$.
Note that even though the condition $B<\log \log T$ 
is not required in the proof of the above lemma, 
our lower bound becomes vacuous once $B \geq \log \log T$ as it becomes smaller than $\sqrt{T}$.

\subsection{Proof of \Lem{lem:prob_en}}

\label{sec:prob_en}
Let $L = \frac{\epsilon^2}{6\Delta^2}$, $\rv{J}_s$ be the random variable for arm $j_s$ pulled in trial $s$,
and $\rv{R}_s$ be the random variable for the reward $r_s$ observed in trial $s$, for $s \in[\sigma]$.
For ease of calculation, we will expand the execution history beyond its stopping time $\sigma$
and let $\rv{J}_s = 0$
and $\rv{R}_s = \half$ for $s \in \{\sigma +1, \cdots, T\}$.
Finally for any trial $t\in [T]$, let $\rv{S}_t := \{(\rv{J}_s, \rv{R}_s)\}_{s \in [t]}$ be the sequence of random variable defining the execution history of the algorithm up until trial $t$, and let $\S_t$ be the realization of this sequence up until trial $t$.

We will begin by showing that in the event of insufficient exploration (i.e. when the algorithm stops quickly), little is learned about the identity of the best arm. In other words, the mutual information between the random variables 
$\bI$ and $\rv{S}_T$ is small.
Using the chain rule for mutual information, we have
\begin{align}
\mi{\bI }{\rv{S}_T} 
		   & = \sum_{t = 1}^T \mi{\bI}{\rv{J}_t \vert \rv{S}_{t-1}} + \mi{\bI}{\rv{R}_t \vert \rv{S}_{t-1}, \rv{J}_t}  \nonumber \\
		   & = \sum_{t = 1}^T 0 + \mi{\bI}{\rv{R}_t \vert \rv{S}_{t-1}, \rv{J}_t}  \tag{$\rv{J}_t$ is deterministic given $\rv{S}_{t-1}$} \nonumber\\	
		   & = \sum_{t=1}^T  \sum_{j \in \A } \sum_{\S_{t-1}} \Pr(\rv{S}_{t-1} = \S_{t-1}, \rv{J}_t = j) \cdot \mi{\bI }{\rv{R}_t \vert \rv{S}_{t-1} = \S_{t-1}, \rv{J}_t = j}. \label{eq:mi_bound}
\end{align}
Using \Fact{fact:kl-info}, we have
\begin{align*}
& \mi{\bI }{\rv{R}_t \vert \rv{S}_{t-1} = \S_{t-1}, \rv{J}_t = j} \\
	 & \qquad = \Ex_{(i^*,y)  \sim \psi }\bracket{\DD{\distribution{\rv{R}_t \mid \rv{S}_{t-1} = \S_{t-1}, \rv{J}_t = j}}{\distribution{\rv{R}_t \mid \bI = i ^*, \rv{S}_{t-1} = \S_{t-1}, \rv{J}_t = j}}}
\end{align*}

We will now prove that the average KL-divergence between the reward distributions for a single pull of an arm under different realizations of instances sampled from our hard distribution $\distD_\A$ is small. Therefore, a single pull of any arm can only provide limited information about the random variables of interest, and therefore, the total information that can be gathered from a small number of pulls is also small.

\begin{claim}
\label{clm:kl_bound}
For any arms $i^*,j \in \A $, trial $t\in [T]$, and any realization $\S_{t-1}$ of the execution history up until trial $t$, we have that
\[
\DD{\distribution{\rv{R}_t \mid \rv{S}_{t-1} = \S_{t-1}, \rv{J}_t = j}}{\distribution{\rv{R}_t \mid \bI = i ^*, \rv{S}_{t-1} = \S_{t-1}, \rv{J}_t = j}}
	\leq 
		6\Delta ^2
			\,.
\]
\end{claim}
\begin{proof}
Let $p = \Pr (\bI  = j \mid \rv{S}_{t-1} = \S_{t-1}, \rv{J}_t = j)$, and let $q=\Pr (\bY  = 1 \mid \bI=j, \rv{S}_{t-1} = \S_{t-1}, \rv{J}_t = j)$.

In the case where $i ^* = j$, it is easy to observe that $\distribution{\rv{R}_t \mid \rv{S}_{t-1} = \S_{t-1}, \rv{J}_t = j} = \B(\half + p q\Delta )$. Moreover, we also have that
$\distribution{\rv{R}_t \mid \bI = i ^*, \rv{S}_{t-1} = \S_{t-1}, \rv{J}_t = j} = \B(\half + q\Delta )$.
We then have that 
\begin{align*}
&\DD{\distribution{\rv{R}_t \mid \rv{S}_{t-1} = \S_{t-1}, \rv{J}_t = j}}{\distribution{\rv{R}_t \mid \bI = i ^*, \rv{S}_{t-1} = \S_{t-1}, \rv{J}_t = j}} \\
	& \qquad \qquad \qquad \qquad = \mathbb{D}\left(\B\left(\half + p q\Delta\right)\mid\mid\B\left(\half + q\Delta\right)\right) \\
	& \qquad \qquad \qquad \qquad  \leq \frac{(\half + p q\Delta - \half - q\Delta)^2}{(\half +q\Delta)(1 - \half - q\Delta )} \\
	& \qquad \qquad\qquad \qquad   \leq  \frac{q^2(1-p)^2\Delta ^2}{\frac{1}{4} - q^2\Delta^2} \\
	&  \qquad \qquad\qquad \qquad  \leq \frac{16q^2(1-p)^2\Delta ^2}{4-q^2} \leq 6 \Delta ^2
		\,,
\end{align*}
where the first inequality above follows from \Fact{fact:kl-chi}, and the final inequality follows due to $\Delta \leq1/4$.

In the case where $i ^* \neq j$, we have that $\distribution{\rv{R}_t \mid \rv{S}_{t-1} = \S_{t-1}, \rv{J}_t = j} = \B(\half + pq \Delta)$.
However, $\distribution{\rv{R}_t \mid \bI = i ^*, \rv{S}_{t-1} = \S_{t-1}, \rv{J}_t = j} = \B(\half)$. Using the same argument as above
\begin{align*}
&\DD{\distribution{\rv{R}_t \mid \rv{S}_{t-1} = S_{t-1}, \rv{J}_t = j}}{\distribution{\rv{R}_t \mid \bI = i ^*, \rv{S}_{t-1} = S_{t-1}, \rv{J}_t = j}} \\
	& \qquad \qquad \qquad \qquad = \mathbb{D}\left(\B\left(\half + p q\Delta\right)\mid\mid\B\left(\half\right)\right) \\
	& \qquad \qquad \qquad \qquad \leq \frac{(\half + pq \Delta - \half)^2}{(\half)(1 - \half)} \\
	& \qquad \qquad \qquad \qquad \leq  4\Delta ^2
		\,.
\end{align*}
\end{proof}
Using \Eqn{eq:mi_bound} and \Clm{clm:kl_bound}  we have that
\begin{align*}
\mi{\bI }{\rv{S}_T} 
		   & \leq \sum_{t=1}^T  \sum_{j \in \A } \sum_{\S_{t-1}} \Pr(\rv{S}_{t-1} = \S_{t-1}, \rv{J}_t = j) \cdot 6 \Delta ^2 \\
		   & = \sum_{t=1}^T \sum_{j \in \A}  \Pr( \rv{J}_t = j) \cdot 6 \Delta ^2 \\
		   & = \sum_{t=1}^T \Pr(\rv{J}_t \neq 0) \cdot 6 \Delta ^2  \\
		   & = \Ex[\sigma ] \cdot 6 \Delta ^2 \\
		   & \leq L \cdot 6 \Delta ^2 = \epsilon^2
		   	\,.
\end{align*}
This implies that that the conditional entropy of $\bI $ given $\rv{S}_T$  is at least
\begin{align*}
\en{\bI  \mid \rv{S}_T} &= \en{\bI} - \mi{\bI }{\rv{S}_T} \geq \en{\bI } - \epsilon^2
\end{align*}

We shall use an analogous argument to bound the mutual information between $\bY$ and $\rv{S}_T$ conditioned on $\bI$.
Using the chain rule for mutual information, we have
\begin{align}
\mi{\bY }{\rv{S}_T\mid \bI}
		   & = \sum_{t = 1}^T \mi{\bY }{\rv{J}_t \vert \rv{S}_{t-1},\bI} + \mi{\bY }{\rv{R}_t \vert \rv{S}_{t-1}, \rv{J}_t,\bI}  \nonumber \\
		   & = \sum_{t = 1}^T 0 + \mi{\bY }{\rv{R}_t \vert \rv{S}_{t-1}, \rv{J}_t,\bI}  \tag{$\rv{J}_t$ is deterministic given $\rv{S}_{t-1}$} \nonumber\\	
		   & = \sum_{t=1}^T  \sum_{i \in \A }\sum_{j\in \A} \sum_{\S_{t-1}} \Pr(\rv{S}_{t-1} = \S_{t-1}, \rv{J}_t = j, \bI=i) \cdot \mi{\bY }{\rv{R}_t \vert \rv{S}_{t-1} = \S_{t-1}, \rv{J}_t = j,\bI=i} \label{eq:mi_bound_y}
\end{align}

We now calculate an upper bound on $\mi{\bY }{\rv{R}_t \vert \rv{S}_{t-1} = S_{t-1}, \rv{J}_t = j,\bI=i}$. 
Using \Fact{fact:kl-info}, we have 
\begin{align*}
&\mi{\bY }{\rv{R}_t \vert \rv{S}_{t-1} = \S_{t-1}, \rv{J}_t = j,\bI=i}\\ 
	 &\quad = \Ex_{y  \sim \psi\mid \bI=i }\bracket{\DD{\distribution{\rv{R}_t \mid \rv{S}_{t-1} = \S_{t-1}, \rv{J}_t = j,\bI=i}}{\distribution{\rv{R}_t \mid \bY = y, \rv{S}_{t-1} = \S_{t-1}, \rv{J}_t = j,\bI=i,\bY=y}}}
\end{align*}

We now have an analogous claim, bounding the KL-divergence between the reward profiles of a single pull of an arm.

\begin{claim}
\label{clm:kl_bound_y}
For any arms $i,j \in \A $, $y\in \{0,1\}$, trial $t\in [T]$, and any realization $\S_{t-1}$ of the execution history up until trial $t$, we have that 
\[
\DD{\distribution{\rv{R}_t \mid \rv{S}_{t-1} = \S_{t-1}, \rv{J}_t = j,\bI=i}}{\distribution{\rv{R}_t \mid \bY = y, \rv{S}_{t-1} = \S_{t-1}, \rv{J}_t = j,\bI=i,\bY=y}}
	\leq 
		6\Delta ^2
			\,.
\]
\end{claim}
\begin{proof}
We begin with the simple case, when $i\neq j$. In this case, it is easy to observe that for any realization of $\bY$, both the reward distributions will be $\B(1/2)$, due to which the KL Divergence will be 0. In the case that $i=j$, let $q=\Pr (\bY  = 1 \mid S_{t-1} = \S_{t-1}, \rv{J}_t = j,\bI=i)$. 

We will first prove this bound in the case that $y=1$.
It easy to observe that $\distribution{\rv{R}_t \mid \rv{S}_{t-1} = \S_{t-1}, \rv{J}_t = j,\bI=i} = \B(1/2 + q\Delta )$. Moreover, we also have that
$\distribution{\rv{R}_t \mid \rv{S}_{t-1} = \S_{t-1}, \rv{J}_t = j,\bI=i,\bY = y} = \B(1/2 + \Delta )$.
We then have that 
\begin{align*}
\DD{\distribution{\rv{R}_t \mid \rv{S}_{t-1} = \S_{t-1}, \rv{J}_t = j,\bI=i}&}{\distribution{\rv{R}_t \mid \bY = 1, \rv{S}_{t-1} = \S_{t-1}, \rv{J}_t = j,\bI=i,\bY=y}}\\
	& = \DD{\B(1/2 + q\Delta)}{\B(1/2 + \Delta)} \\
	& \leq \frac{(\half + q\Delta - \half - \Delta)^2}{(\half +\Delta)(\half - \Delta )} \\
	& \leq  \frac{(1-q)^2\Delta ^2}{\frac{1}{4} - \Delta^2} \\
	& \leq \frac{16(1-q)^2\Delta ^2}{3} \leq 6 \Delta ^2
		\,,
\end{align*}
where the first inequality above follows from \Fact{fact:kl-chi}, and the final inequality follows due to $\Delta <1/4$.

In the case that $y=0$,
we again have $\distribution{\rv{R}_t \mid \rv{S}_{t-1} = \S_{t-1}, \rv{J}_t = j,\bI=i} = \B(1/2 + q \Delta)$.
However, $\distribution{\rv{R}_t \mid \rv{S}_{t-1} = \S_{t-1}, \rv{J}_t = j,\bI=i,\bY=y} = \B(1/2)$. Using the same argument as above
\begin{align*}
\DD{\distribution{\rv{R}_t \mid \rv{S}_{t-1} = \S_{t-1}, \rv{J}_t = j,\bI=i}&}{\distribution{\rv{R}_t \mid \bY = 1, \rv{S}_{t-1} = \S_{t-1}, \rv{J}_t = j,\bI=i,\bY=y}}\\
	& = \DD{\B(1/2 + q\Delta)}{\B(1/2)} \\
	& \leq \frac{(\half + q \Delta - \half)^2}{(\half)(1 - \half)} \\
	& \leq  4\Delta ^2
		\,.
\end{align*}
\end{proof}

Using \Eqn{eq:mi_bound_y} and \Clm{clm:kl_bound_y}  we have that
\begin{align*}
\mi{\bY }{\rv{S}_T|\bI} 
		   & \leq \sum_{t=1}^T  \sum_{i \in \A }\sum_{j\in \A} \sum_{\S_{t-1}} \Pr(\rv{S}_{t-1} = \S_{t-1}, \rv{J}_t = j, \bI=i) \cdot 6 \Delta ^2 \\
		   & = \sum_{t=1}^T \sum_{j \in \A}  \Pr( \rv{J}_t = j) \cdot 6 \Delta ^2 \\
		   & = \sum_{t=1}^T \Pr(\rv{J}_t \neq 0) \cdot 6 \Delta ^2  \\
		   & = \Ex[\sigma ] \cdot 6 \Delta ^2 \\
		   & \leq L \cdot 6 \Delta ^2 = \epsilon^2
		   	\,.
\end{align*}

As before, we use this upper bound on the mutual information between $\bY$ and $\rv{S}_T$ given $\bI$ to lower bound the conditional entropy of $\bY$ given $\rv{S}_T$ and $\bI$ as
\[
\en{\bY  \mid \rv{S}_T,\bI} = \en{\bY \mid \bI} - \mi{\bY }{ \rv{S}_T\mid \bI} \geq \en{\bY \mid \bI} - \epsilon^2
	\,.
\]

These bounds demonstrate that in expectation, the entropies of the posterior distributions of $\bI$ and $\bY|\bI$ given the samples drawn by the algorithm will remain large if the algorithm does not draw sufficiently many samples. We shall further show that this must necessarily be the case, not just in expectation, but also with high probability.

Consider any realization  $\S_T$ for $\rv{S}_T$. 
We say that the realized outcome profile $\S_T$  is $\epsilon$-\emph{uninformative} iff both, $\en{\bI  \mid \rv{S}_T = \S_T} \geq \en{\bI } - \epsilon$, and $\en{\bY  \mid \bI, \rv{S}_T = \S_T} \geq \en{\bY\mid \bI } - \epsilon$. Roughly speaking, whenever the outcome profile $\S_T$ is $\epsilon$-uninformative, the algorithm 
	is quite ``uncertain'' about both, the identity of $i ^*$, as well as its reward $\mu_{i^*}$ (controlled through the variable $y$) and hence needs to estimate both among a large pool of possibilities in a later pass. 
To show that a realized outcome profile $\S_T$ will be $\epsilon$-uninformative with high probability, let $C_I:= \en{\bI} - \en{\bI  \mid \rv{S}_T}$. 
By Markov's inequality, we have that
		\begin{align*}
			\Pr_{\S_T} \paren{\en{\bI}- \en{\bI  \mid \rv{S}_T = \S_T} \geq  C_I/\epsilon} &\leq \frac{\Ex_{\S_T}\bracket{\en{\bI} - \en{\bI  \mid \rv{S}_T = S_T}}}{C_I/\epsilon} \\
			&= \frac{  \paren{\en{\bI} - \en{\bI  \mid \rv{S}_T}}}{C_I/\epsilon} = \epsilon \tag{by the choice of $C_I = \en{\bI} - \en{\bI  \mid \rv{S}_T}$}
		\end{align*}
Following an identical calculation with $C_Y:= \en{\bY\mid\bI} - \en{\bY  \mid\bI, \rv{S}_T}$, we have
\[\Pr_{\S_T} \paren{\en{\bY\mid\bI}- \en{\bY  \mid \bI,\rv{S}_T = \S_T} \geq  C_Y/\epsilon} \leq \epsilon\]
Since both, $C_I,C_Y \leq \epsilon^2$, we have with probability at least $1-2\epsilon$ over realizations $\S_T$ of $\rv{S}_T$,  
\begin{equation}
\label{eq:posterior_entropy_I}
	\en{\bI} - \en{\bI  \mid \rv{S}_T = \S_T} < \frac{C_I}{\epsilon} \leq \epsilon  \implies  \en{\bI  \mid \rv{S}_T = \S_T} \geq \log{{ A}} - \gamma - \epsilon
		\,,
\end{equation}
as well as
\begin{equation}
\label{eq:posterior_entropy_Y}
	\en{\bY\mid \bI} - \en{\bY  \mid\bI, \rv{S}_T = \S_T} < \frac{C_Y}{\epsilon} \leq \epsilon  \implies  \en{\bY  \mid\bI, \rv{S}_T = \S_T} \geq \log{{ 2}} - \gamma - \epsilon
		\,.
\end{equation}
Henceforth, we shall use $\Tuni$ to refer to an $\epsilon$-uninformative realization of $\rv{S}_T$.

Now fix any $\epsilon$-uninformative realization $\Tuni$. Let $\M  \subset \A $ be the set of arms of size $|\M| = \beta A$ chosen to be retained by the algorithm given its execution history $\Tuni$.
and let $\R = \A\setminus \M $ denote the remaining set of rejected arms.  
Using \Lem{lem:bounded_mass} we can argue that
\begin{equation}
\label{eq:prob_discarded}
\Pr(\rv{I}  \in \M  \mid \rv{S}_T = \Tuni) \leq \log\left(1+\beta \right)+\gamma_{I} + \epsilon \implies \Pr(\rv{I}  \in \R  \mid \rv{S}_T = \Tuni) > 1 - \log\left(1+\beta\right)- \gamma_{I} - \epsilon
	\,.
 \end{equation}
We will finally prove that in the event that the sequence of rewards observed by the algorithm was uninformative, and the algorithm actually did end up discarding the best arm from its memory, then the entropy of the identity of the best arm remains large amongst the arms the algorithm chose to reject at its stopping time. 

\begin{claim}
\label{clm:en_bound}
For any $\epsilon$-uninformative realization $\Tuni$, we have
\[
\en{\bI  \mid \rv{S}_T = \Tuni, \rv{I}  \notin \M  } \geq \log ((1-\beta)A)-  \frac{\gamma+\epsilon}{1-\log(1+\beta)-\gamma-\epsilon}
	\,.
\]
\end{claim}
\begin{proof}
Suppose, for the sake of contradiction, that the above inequality is not true.
Let $\rv{X}$ be an indicator random variable which is $1$ when $\rv{I} \notin \M$, and $0$
otherwise. Furthermore,  let $p=\Pr(\rv{X} = 1|\rv{S}_T = \Tuni) = \Pr(\bI\notin \M\mid \rv{S}_T = \Tuni)$.
Then we have that 
\begin{align*}
\en{\bI \mid \rv{S}_T = \Tuni} & \overset{(a)}{\leq} \en{\bI, \rv{X} \mid \rv{S}_T = \Tuni} \\
		& \overset{(b)}{=} \en{\rv{X} \mid \rv{S}_T = \Tuni}  + \en{\bI \mid \rv{X} , \rv{S}_T = \Tuni} \\
		& = p\left(\en{\bI \mid \rv{S}_T = \Tuni, \rv{X}=1} +\log \frac{1}{p}\right) + (1-p)\left(\en{\bI \mid \rv{S}_T = \Tuni, \rv{X}=0}+\log \frac{1}{1-p}\right)\\
&\overset{(c)}{<} p\left(\log ((1-\beta)A) -  \frac{\gamma+\epsilon}{1-\log(1+\beta)-\gamma-\epsilon}+ \log \frac{1}{p}\right) + (1-p)\left(\log (\beta A)+ \log \frac{1}{1-p}\right)\\
&= \log A -\frac{p(\gamma+\epsilon)}{1-\log(1+\alpha)-\gamma-\epsilon} + p\log\frac{1-\beta}{p} +(1-p)\log\frac{\beta}{1-p} \\
& \overset{(d)}{\leq} \log A - \frac{p(\gamma+\epsilon)}{1-\log(1+\alpha)-\gamma-\epsilon} + \log \left(p\cdot \frac{(1-\beta)}{p} + (1-p)\cdot\frac{\beta}{1-p}\right)\\
&< \log A - \gamma - \epsilon,
\end{align*}
where $(a)$ follows due to the fact that the joint entropy in $(\bI, \rv{X})$ is at least the entropy in $\bI$,
$(b)$ follows due to the chain rule for entropy, $(c)$ follows by our assumption (for the sake of contradiction),
 $(d)$ follows by Jensen's inequality, and the final inequality follows from bounding $p$ through Equation~\ref{eq:prob_discarded}. This contradicts the bound achieved in Equation~\ref{eq:posterior_entropy_I}.
\end{proof}

We further argue a similar claim about the uncertainty in estimating the reward of the best arm in the event that the sequence of rewards observed by the algorithm is uninformative, and the algorithm did end up discarding the best arm from its memory.

\begin{claim}
\label{clm:en_bound}
For any $\epsilon$-uninformative realization $\Tuni$, we have that
\[
\en{\bY  \mid \bI,\rv{S}_T = \Tuni, \rv{I}  \notin \M  } \geq \log 2 -  \frac{\gamma+\epsilon}{1-\log(1+\beta)-\gamma-\epsilon}
	\,.
\]
\end{claim}
\begin{proof}
Suppose, for the sake of contradiction, that the above inequality is not true. Let $\rv{X}$ be a random variable that takes value $1$ when $\bI\notin \M$, and $0$ otherwise, and let $p = \Pr(\rv{X}=1\mid \rv{S}_T = \Tuni) = \Pr(\bI \notin \M\mid \rv{S}_T = \Tuni)$. We have that 
\begin{align*}
\en{\bY  \mid \bI,\rv{S}_T = \Tuni } &\leq \en{\bY, \rv{X}  \mid \bI,\rv{S}_T = \Tuni } \\
	&= \en{\bY \mid \bI,\rv{S}_T = \Tuni, \rv{X}  } + \en{ \rv{X}  \mid \bI,\rv{S}_T = \Tuni } \\
& \overset{(a)}{=} \en{\bY \mid \bI,\rv{S}_T = \Tuni, \rv{X}  }  \\
&= p\en{\bY\mid \bI,\rv{S}_T=\Tuni,\rv{X}=1} + (1-p)\en{\bY\mid \bI,\rv{S}_T=\Tuni,\rv{X}=0}
\end{align*}
where $(a)$ follows by observing that upon conditioning on the identity of the best arm $\bI$, as well as the observed outcome profile $\rv{S}_T$, the value of the random variable $\rv{X}$ (i.e. whether the best arm was retained or discarded) is fixed, since the algorithm is deterministic. Therefore, $\en{\rv{X}\mid \bI,\rv{S}_T=\Tuni } = 0$. We now have
\begin{align*}
\en{\bY  \mid \bI,\rv{S}_T = \Tuni } &\leq 
	p \en{\bY  \mid \bI, \rv{S}_T = \Tuni, \rv{X}=1}  + (1-p) \en{\bY  \mid \bI, \rv{S}_T= \Tuni, \rv{X}= 0} \\
	& <  p\left( \log 2 -  \frac{\gamma+\epsilon}{1-\log(1+\beta)-\gamma-\epsilon}\right) + (1-p) \log 2\\
	& < \log 2 - \gamma - \epsilon  
	\,,
\end{align*}
where the final inequality follows from bounding $p$ through Equation~\ref{eq:prob_discarded}, which contradicts the bound achieved in Equation~\ref{eq:posterior_entropy_Y}.
\end{proof}

We finally show that this outcome is not a rare event, but rather quite likely
\[
\Pr(\bI \notin \M) \geq 1 - \Pr(\S_T \text{ is informative})  -  \Pr(\bI \in \M \mid \S_T \text{ is uninformative})
	\geq 1 - \log(1+\beta) - \gamma - 3\epsilon
		\,.
\]
This completes the proof of \Lem{lem:prob_en}.

\subsection{Proof of \Lem{lem:multi_pass_lower}}
\label{subsec:multi_pass_lower}

We will prove this lemma using induction on the number of passes $b$.
Recall that the $j$-th epoch begins when the first arm of the $j$-th layer is read into memory and ends right before the first arm of the $j+1$-th layer is read into memory.
Also, recall that the total memory budget of the algorithm is $\beta K/(b+1)$
and the number of arms in each layer is $K/(b+1)$.

We consider a modified setting where the algorithm is allowed additional power, i.e.\
it is allowed to store all arms from layer $j$ in memory during the execution of epoch $j$. However, at the end of epoch $j$ the algorithm needs to throw at least $1-\beta$
fraction of these layer $j$ arms
and is only allowed to retain in memory at most a $\beta$ fraction of the arms from that epoch in addition to the arms stored from previous epochs.\footnote{Without loss of generality, we can assume the algorithm retains exactly a $\beta$ fraction of the arms as it can always choose to ignore the extra arms.}
Under this modified setting the algorithm can even use up to $\beta K$ memory but is constrained to storing no more than a $\beta$ fraction of arms from any single epoch.
This cannot hurt the regret as we are only allowing more memory, which can always be ignored. Formally, any algorithm that uses at most $\beta K/(B+1)$ memory (where $\beta = (8b\log e)^{-1}$) in the original setting
can be used in this modified setting
as it is allowed to use strictly more memory for each epoch in the 
modified setting.
Also, note that the algorithm incurs the same regret in both settings. 
Hence, an optimal algorithm in this 
modified setting cannot incur more regret than an optimal 
algorithm in the original setting. Let $\alpha = 2^B/(2^{B+1}-1)$.

\textbf{Base Case ($b=1$):} Let  $\epsilon = 1/288$, $L = \epsilon^{2}/(6\Delta_1^2)$, and $\sigma$ be the (random) length (number of trials) of the first epoch.

\textbf{Case 1.} $[\Exp[\sigma] \geq L]$, with the expectation taken over the observations made by the algorithm.

In this case, we claim that the algorithm will suffer an expected regret $\Omega(L\Delta_2)$. To see this, observe that $\en{\bY_2} \geq \en{\bY_2|\bI_2} \geq \log 2 - \gamma$, which follows from \Fact{fact:it-facts}. Therefore, by \Lem{lem:bern_en}, we have that the random variable $\bY_2$ is distributed as a Bernoulli $\B(p)$ with parameter $p$ such that $|p-\half| \leq \sqrt{5\ln (4) \gamma/16} = \sqrt{(5\ln 4)/2}/16$, which follows from the fact that $\gamma < 1/32$. Therefore we have that the best arm $i^*_2$ will realize to have a large reward $\mu_{i^*_2} = 1/2+\Delta_2$ with probability at least $1/2 - \sqrt{(5\ln 4)/2}/16$, giving us that the expected regret of the algorithm
\begin{align*}
\Exp[R_T(\ALG)] &\geq \left(\half-\frac{\sqrt{(5\ln 4)/2}}{16}\right) L\cdot\Delta_2\\ 
&= \left(\half-\frac{\sqrt{(5\ln 4)/2}}{16}\right) \frac{\epsilon^2}{6}\cdot\frac{\Delta_2}{\Delta_1^2}\\
&= \Omega\left(T^{\frac{2^{B+1}-2}{2^{B+1}-1}} \cdot T^{-\frac{2^{B} - 2}{2^{B+1}-1}}\right) = \Omega\left(T^{\alpha}\right)
\end{align*}

\textbf{Case 2.} $[\Exp[\sigma] < L]$

Let $\S_{\sigma}$ be the outcomes observed by the algorithm over the arms sampled in epoch $1$. Then by \Lem{lem:prob_en}, we have that after observing the outcomes $\S_{\sigma}$, the best arm $i^*_1$ will be discarded by the algorithm (i.e. $\bI_1\notin \M$ where $\M$ is the set of arms from epoch $1$ that are retained in memory by the algorithm), and the entropy of the posterior distribution 
\begin{align*}
\en{\bY_1|\S_{\sigma},\bI_1,\bI_1\notin \M} &\geq \log 2 - \frac{\gamma+\epsilon}{1-\left(\log\left(1+\beta\right) + \gamma + \epsilon\right)}\\
&= \log 2 - \frac{\gamma+\epsilon}{1-\left(\log\left(1+\frac{1}{8\log e}\right) + \gamma + \epsilon\right)}\\
&\geq \log 2 - \frac{\gamma+\epsilon}{1-\left(\frac{\log e}{8\log e} + \gamma + \epsilon\right)}\\
&= \log 2 - \frac{10}{242}
\end{align*}
 with probability at least $1-\left(\log\left(1+\beta\right)-\gamma-3\epsilon\right) \geq \frac{5}{6}$. Furthermore, since we have that $\Exp(\sigma)<L = \epsilon^2/(6\Delta_1^2)$, by Markov's inequality, the actual number of trials $\sigma$ spent in epoch $1$ will be at most $1/\Delta_1^2 = o(T)$ with probability at least $1-\epsilon^2/6 \geq 1-10^{-6}$. There are least $T-1/\Delta_1^2=T-o(T) = \Omega(T)$ trials left in epoch $2$ with a very high constant probability. In this case, the algorithm will suffer large regret when the best arm $i^*_1$ in epoch $1$ realizes to have a large reward $\mu_{i^*_1} = 1/2+\Delta_1$, i.e. $\bY_1$ realizes to have value $1$, and the best arm $i^*_2$ in epoch $2$ realizes to have a low reward of $\mu_{i^*_2} = 1/2$, i.e. $\bY_2$ realizes to have value $0$. 
 
 Observe that in the posterior distribution of the rewards of arms in epoch 1, the entropy in the reward of the best arm in the first epoch $\en{\bY_1|\S_{\sigma},\bI_1\notin \M} \geq \en{\bY_1|\S_{\sigma},\bI_1,\bI_1\notin \M} \geq \log 2 - 10/242$, and therefore, by \Lem{lem:bern_en}, the posterior distribution of $\bY_1$ is Bernoulli with parameter $p\geq 1/2 - \sqrt{(5\cdot \ln 4\cdot 10)/(242\cdot16)} = 1/2 - 5\sqrt{\ln 4}/44$ (a constant bounded away from 0) which is the probability with which the best arm in the first epoch actually had a large reward. Similarly, we have that in the prior distribution of the rewards of arms in epoch 2, the reward of the best arm in the second epoch $\en{\bY_2}\geq \en{\bY_2|\bI_2}\geq \log 2 - \gamma$, and therefore, we have that the distribution of $\bY_2$ is Bernoulli with parameter $p\leq 1/2 + \sqrt{(5\cdot \ln 4)/(32\cdot 16)} = 1/2 + \sqrt{(5\ln 4)/2}/16$. Therefore, the best arm in the second epoch realizes to have low reward with probability at least $1/2 - \sqrt{(5\ln 4)/2}/16$ (a constant bounded away from 0). Therefore, we have that the expected regret of the algorithm in this case 
\begin{align*}
\Exp[R_T(\ALG)] &\geq \frac{5}{6} \cdot\left(\half-\frac{5\sqrt{\ln 4}}{44}\right)\cdot\left(\half - \frac{\sqrt{(5\ln 4)/2}}{16}\right)(1-10^{-6}) \cdot(T-o(T))\cdot\Delta_1\\ &\geq \Omega\left((T-o(T))\cdot  T^{-\frac{2^{B} - 1}{2^{B+1}-1}}\right) = \Omega\left(T^{\alpha}\right)
\end{align*}

Therefore, the expected regret is $\Omega(T^{\alpha})$, which proves the base case.

\textbf{Induction Step:} Assuming the lemma is true for any number of passes up to $b-1$, will show that it also holds for $b$ passes. 
Suppose for the sake of contradiction that the 
claim is not true for $b$, i.e. there exists a $b$-pass algorithm \ALG with memory at most $K(8b(b+1)\log e)^{-1}$ whose expected regret over rewards drawn from the distribution $\dist^{\{\psi_j\}_{j=1}^{b+1}}_{\K,B}$ is $o\paren{T^{\alpha} 4^{-b}}$.

The general outline will again be to show that if the 
algorithm ends any epoch after  
performing \emph{sufficient exploration}, then it will incur 
large regret in the case that arm $i_{b+1}^*$ realizes to a large mean reward, contradicting the assumption that $\ALG$ has small regret. On the other hand, if the algorithm ends all epochs with 
\emph{insufficient exploration}, then the algorithm will not just discard all best arms, but also the instance induced over the discarded arms will remain hard. Supposing the algorithm achieves low expected regret over this instance in $b-1$ passes, it would contradict our induction hypothesis. 

Consider any epoch $j \in [b]$.
Let $\gamma_{+} = 1/(32b)$, $\epsilon = 2\gamma_{+}/(9b)$, and $L_j := \epsilon^2/(6\Delta_j^2)$. For any epoch $j\in [b+1]$, let $t_j$ be the trial when epoch $j$ begins, let $\T_j := \{t_j, t_j+1, \cdots, t_{j+1} -1\}$ be the trials that belong to epoch $j$, and let $\sigma_j = |\T_j|$ denote the number of trials in epoch $j$. Lastly, let $\S_{\sigma_j} = \cup_{r=1}^j \{(i_t, r_t)\}_{t \in \mathcal{T}_r}$ be the sequence of observations defining the execution history of the algorithm until the end of epoch $j$ 
with $i_t$ being the arm pulled and $r_t$ being the reward realized in trial $t$, respectively.

\textbf{Case 1.} $\left[\Exp(\sigma_j) \geq \epsilon^2/(6\Delta^2_j)\text{ for some } j\in [b]\right]$, where the expectation is over realizations of the rewards until epoch $j.$

In this case, observe that the expected regret 
of the algorithm is $\Omega(L_j \cdot\Delta_{b+1})$ in the event where the best arm in the final epoch (that has not been seen yet) realizes to have a large reward, i.e. $\mu_{i_{b+1}^*} = \half + \Delta_{b+1}$. By definition of the input instance and \Fact{fact:it-facts}, we have that $\en{\bY_{b+1}} \geq \en{\bY_{b+1}|\bI_{b+1}} \geq \log 2 - \gamma$. Therefore, by \Lem{lem:bern_en}, we have that $\bY_{b+1}$ is distributed as a Bernoulli $\B(p)$ with parameter $p$ such that $|p-1/2|\leq \sqrt{(5\gamma \ln 4)/16}$. Since $\gamma \leq 1/(32b)$, and $b\geq2$, we have $\sqrt{5\gamma \ln 4/16} \leq \sqrt{(10b^{-1}\ln 4)}/32 < 1/4 - 1/6 - 1/(200b^3)$ for $b\geq 2$. Therefore, we have that $\mu_{i^*_{b+1}} = 1/2 + \Delta_{b+1}$ with probability at least $1/2 - (1/4 - 1/6 - 1/(200b^3))$.

Therefore, 
the expected regret in the event that $\Ex[\sigma_j] > L_j$ for some $j\in [b]$
\[
\Exp[R_T(\ALG)] \geq \Omega(L_j\Delta_{b+1}) = \Omega\left(T^{\frac{2^{B+1} - 2^j}{2^{B+1} - 1}}
		\cdot T^{-\frac{2^{B} - 2^{b}}{2^{B+1} - 1}}\cdot\epsilon^{2}\right)
		= \Omega\left(T^{\frac{2^{B} + 2^{b}-2^{j}}{2^{B+1} - 1}}\cdot \epsilon^2\right)= \Omega(T^\alpha b^{-4})
		\,,
\]
which contradicts the assumption that the expected regret of the algorithm is $o\paren{T^{\alpha}4^{-b}}$.

\textbf{Case 2.} $\left[\Exp(\sigma_j) < \epsilon^2/(6\Delta_j^2) \text{ for all } j\in [b]\right]$

In this case, we will leverage \Lem{lem:prob_en} to show that the algorithm will not be able 
to collect \emph{sufficient information} about $i^*_j$ and it will suffer large regret in the remaining number of passes.
Using \Lem{lem:prob_en} we will show that the conditional distribution for $\bI_j$
after epoch $j$ will have high entropy.
Let $\mathcal{T}_j \subseteq [T]$ be the trials that belong to epoch $j$.
Let $\S_{\sigma_j} = \{(i_t, r_t)\}_{t \in \mathcal{T}_j}$ be the execution history of epoch $j$ 
with $i_t$ being the arm pulled and $r_t$ being the reward realized in trial $t$, respectively.
Given $\S_{\sigma_j}$, let $\M_j\subset \A_j$ be the set of $|\M_j|=\beta A = K/(8b(b+1)\log e)$ arms retained by the algorithm in memory, and let $\R_j = \A_j \setminus \M_j$ be the set of $|\R_j| = (1-\beta)A = R$ arms that were rejected after epoch $j$.
Let $\gamma_{-} = 1/(32(b-1)) = b\gamma_{+}/(b-1)$. We first observe that by \Lem{lem:prob_en}, the entropy of the posterior
\begin{align*}
 \en{\bI_{j}|\S_{\sigma_j},\bI_j\in \R_j} =\en{\bI_{j}|\S_{\sigma_j},\bI_j\not\in \M_j} &\geq \log R - \frac{\gamma+\epsilon}{1-\left(\log\left(1+\beta\right) + \gamma + \epsilon\right)}\\
&\geq \log R - \frac{\gamma+\epsilon}{1-\left(\log\left(1+\frac{1}{8b\log e}\right) + \gamma + \epsilon\right)}\\
&\geq \log R + \frac{\left(1+\frac{2}{9b}\right)\gamma_{+}}{1-(\frac{\log e}{8b\log e} + \gamma_{+} + \frac{2\gamma_{+}}{9b})}\\
&= \log R -\underbrace{\frac{\left(1+\frac{2}{9b}\right)}{1-(\frac{1}{8b} + \frac{1}{32b}+ \frac{1}{144b^2})}}_{x}\gamma_{+}.
\end{align*}
We claim that $x \leq b/(b-1)$, which would imply that the entropy of the posterior $\en{\bI_{j}|S_{\sigma_j},\bI_j\in E_j} \geq \log |E_j| - \gamma_{-}$. We have 
\begin{align*}
x &= \frac{\left(1+\frac{2}{9b}\right)}{1-(\frac{1}{8b} + \frac{1}{32b}+ \frac{1}{144b^2})}\\
&< \frac{1+\frac{2}{9b}}{1-(5 + \frac{2}{9}) (\frac{1}{32b})}\\
&= \frac{288b+64}{288b-47}\\
&< \frac{b}{b-1},
\end{align*}
where the final inequality follows by observing $\frac{288b+64}{288b-47} - \frac{b}{b-1} = \frac{-(177b+64)}{(b-1)(288b-47)} < 0$. The proof of the fact that $\en{\bY_{j}|\S_{\sigma_j},\bI_j,\bI_j\in \R_j} \geq \log 2 - \gamma_{-}$ follows by the exact same calculation. 

At this point, we further argue that this supposed low regret algorithm $\ALG$ cannot spend too many trials on the first $b$ epochs prior to processing the $(b+1)^{th}$ epoch with a large probability. Since we have that $\Ex(\sum_{j\in [b]}\sigma_j) < \sum_{j\in [b]} L_j = \epsilon^2 \sum_{j\in [b]} (6\Delta_j^2)^{-1} $, by Markov's inequality, it must be that $\Pr(\sum_{j\in [b]}\sigma_j \geq \sum_{j\in [b]}\Delta^{-2}_j) \leq \epsilon^2/6 < (10b)^{-4}$. We define the event $\event_0 := (\sum_{j\in [b]}\sigma_j < \sum_{j\in [b]}\Delta^{-2}_j)$ where the actual number of trials spent by the algorithm in the first $b$ epochs is small, and the above calculation gives us that $\Pr(\neg \event_0) < (10b)^{-4}$. Now for every $j\in [b]$, let us define the event

\begin{align*}
\event_j := \paren{ \bI_j \in \R_j \text{ and } \en{\bI_j \mid \S_{\sigma_j}, \bI_j \in \R_j} \geq \log R - \gamma_{-} \text{ and } \en{\bY_j\mid \S_{\sigma_j}, \bI_j, \bI_j \in \R_j} \geq \log 2 - \gamma_{-}}
	\,.
\end{align*}
Using \Lem{lem:prob_en}, we have that 
\begin{align*}
\Pr(\neg \event_j) &\leq \log\left(1+\frac{1}{8b\log e}\right) + \gamma + 3\epsilon\\
&\leq \frac{\log e}{8b\log e} + \gamma + \frac{6\gamma}{9b}\\
&= \frac{1}{8b} + \frac{1}{32b} + \frac{1}{96b} = \frac{1}{6b} 
	\,.
\end{align*}
where the final inequality follows by observing $b\geq 2$. 

Let us also define the event $\event_{b+1} := \{\mu_{i^*_{b+1}} = 1/2\}$. By definition of the input instance and \Fact{fact:it-facts}, we have that $\en{\bY_{b+1}} \geq \en{\bY_{b+1}|\bI_{b+1}} \geq \log 2 - \gamma$. Therefore, by \Lem{lem:bern_en}, we have that $\bY_{b+1}$ is distributed as a Bernoulli $\B(p)$ with parameter $p$ such that $|p-1/2|\leq \sqrt{(5\gamma \ln 4)/16}$. Since $\gamma \leq 1/(32b)$, and $b\geq2$, we have $\sqrt{5\gamma \ln 4/16} \leq \sqrt{5\ln 4}/32 < 1/4 - 1/6 - 1/(200b^3)$. Therefore, we have that $\Pr(\neg \event_{b+1}) \leq 1/2 + (1/4-1/6 -1/(200b^3))$. Therefore, by a union bound over all these $(b+2)$ events, we have
\[\Pr(\event_0\cap_{j\in [b+1]} \event_j) \geq 1 - \frac{1}{(10b)^4}- \frac{b}{6b} - \left(\frac{1}{2} + \frac{1}{4} - \frac{1}{6} -\frac{1}{200b^3}\right) \geq \frac{1}{4}\]

Lastly, we argue that under event $\event$, the algorithm must necessarily spend $o(T)$ trials in the last epoch $b+1$. This is because under event $\event$, $\mu^*_{b+1} = 1/2$, and furthermore, in the posterior distribution of the reward of the best arm $\mu_{i^*_b}$ of the $b^{th}$ epoch in the rejected set $\R_b$ is at least $\en{\bY_b|\S_{\sigma_b},\bI_b\notin \M_b} \geq \en{\bY_b|\S_{\sigma_b},\bI_b,\bI_b\notin \M_b} \geq \log 2 - \gamma_-$, where $\gamma_- \leq 1/(32(b-1))$. Therefore, by \Lem{lem:bern_en}, the posterior distribution of $\bY_b$ is Bernoulli with parameter $p\geq 1/2 - \sqrt{(5\cdot \ln 4\cdot \gamma_-)/16} = 1/2 - \sqrt{(5\ln 4)/2}/16 > 1/2 - 1/8$, which is the probability with which the rejected best arm in the $b^{th}$ epoch actually had a large reward. Therefore, if the algorithm spends $\Omega(T)$ trials in the $(b+1)^{th}$ epoch, then the expected regret of the algorithm
\begin{align*}
\Exp[R_T(\ALG)] \geq \Pr(\event)\Exp[R_T(\ALG)|\event] &\geq \frac{1}{4}\cdot \frac{3}{8}\cdot\Omega(T)\cdot\Delta_b\\ 
&= \Omega\left( T\cdot T^{-\frac{2^{B} - 2^{b-1}}{2^{B+1} - 1}} \right)\\
&=\Omega\paren{T^{\frac{2^{B} + 2^{b-1} -1 }{2^{B+1} - 1}}} = \Omega\paren{T^{\alpha}},
\end{align*}
contradicting our assumption about the expected regret achieved by the algorithm. Therefore, under event $\event$, we have that the total number of trials spent by the algorithm in the first pass is necessarily $\sum_{j\in [b+1]}\sigma_j = o(T) +\sum_{j\in [b]} \Delta_j^{-2} = o(T) +\sum_{j\in [b]} T^{1-\frac{2^j-1}{2^{B+1}-1}} =o(T)$. Therefore, the number of trials $T_-$ left is necessarily $T_- = T-o(T)$, and only $b-1$ passes left. 

Now we shall use our assumption about the expected regret achievable by our algorithm to prove that in order to achieve low expected regret overall, it must necessarily achieve low expected regret in the remaining passes too. Let $\Exp[R_{T_-}(\ALG)]$ denote the cumulative regret of the algorithm over the remaining $b-1$ passes. Therefore, we have 
\[
\Exp[R_T(\ALG)]\geq \Ex[R_{T_-}(\ALG) \mid  \event ] \cdot \Pr(\event)
	= \frac{1}{4} \cdot \Ex[R_{T_-}(\ALG) \mid \event]
	\,.
\]
Therefore, $\Ex[R_{T_-}(\ALG) \mid \event] \leq 4 \Ex[R_T(\ALG)] = o\left(4\cdot T^{\alpha}\cdot 4^{-b}\right) = o\left(T^{\alpha}\cdot 4^{-(b-1)}\right)$. We shall use this fact to set up a contradiction to our induction hypothesis, which at a high level says that any $(b-1)$-pass algorithm with small memory must incur large regret. We begin by setting up the hard distribution.

In our new instance, we begin by discarding the arms $\A_{b+1}$ as under event $\event$, the reward of the best arm (and hence all arms) in this epoch $(b+1)$ has realized to a low value. Our new instance consists of all the rejected arms $\R_j$ for $j\in [b]$, the first $b$ epochs. We refer to these set of arms as $\K' = \cup_{j\in [b]}\R_j$, whose size is exactly $|\K'| = K = \sum_{j\in [b]} (1-\beta)K/(b+1) = (1-\beta)bK/(b+1)$.

Next, we claim that the posterior distributions over $\bI_j,\bY_j$ in $\R_j$ for $j\in [b]$ to give us a hard distribution over arms $\K'$ for the $(b-1)$ pass algorithm. For $j\in [b]$, let $\phi_j$ be the joint distribution $\distribution{(\bI_j,\bY_j) \mid \bI_j \in \R_j}$, and let $\dist^{\{\phi_j\}_{j=1}^b}_{\K',B}$. Its easy to verify that $\dist^{\{\phi_j\}_{j=1}^b}_{\K',B}$ satisfies the requirements for a hard distribution for a $(b-1)$ pass algorithm, as the partitions $\{\R_j\}$ of $\K'$ are all of equal size $|\K'|/b$, and for random variables $(\bI_j,\bY_j)\sim \phi_j$ for any $j\in [b]$ satisfy the high entropy condition $\en{\bI_j} \geq \log |\R_j| - \gamma_-$, and $\en{\bY_j|\bI_j} \geq \log 2 - \gamma_-$, where $\gamma_- \leq (32(b-1))^{-1}$ (as indicated by event $\event$). Furthermore, the memory budget for a $(b-1)$ pass algorithm for this instance is 
\[\frac{K'}{8b(b-1)\log e} = \frac{(1-\beta)bK}{b+1}\cdot \frac{1}{8b(b-1)\log e} > \frac{\left(1-\frac{1}{b}\right)bK}{b+1}\cdot \frac{1}{8b(b-1)\log e} = \frac{K}{8b(b+1)\log e},\]
which is in fact larger than the memory used by $\ALG$. We shall show that we can use the behavior of $\ALG$ in the subsequent $b-1$ passes under event $\event$ to construct a $(b-1)$-pass algorithm with low memory that achieves $o(T^{\alpha}\cdot 4^{-(b-1)})$ expected regret on the above hard instance, which contradicts our induction hypothesis. 

Let $\ALG_-$ denote the algorithm $\ALG$ for the remaining $b-1$ passes when event $\event$ occurs.
Under the assumption that the regret $R_{T_-}(\ALG)$ is small, we will 
construct a $(b-1)$-pass algorithm $\ALG_{b-1}$ with small memory that achieves small regret over time horizon $T_-=T-o(T)$ on the 
instance $\dist_{\K', B}^{\{\phi^j\}_{j=1}^{b}}$. $\ALG_{b-1}$ is constructed as follows:
if $\ALG_-$ pulls an arm in $\cup_{j \in [b]} ~ \supp{\phi_j} $, $\ALG_{b-1}$ also pulls the corresponding arm in $ \K'$ 
and returns the realized reward of the arm to $\ALG_{-}$; otherwise,
$\ALG_{b-1}$ simply pulls an arm with distribution $\B(1/2)$ and returns the result to $\ALG_{-}$. 
		
It is trivially true that $\Ex[R_{T_-}(\ALG_{b-1})] \leq \Ex[R_{T_-}(\ALG_-) \mid \event]$. 
This is because, given event $\event$, any other arm than the arm in $\cup_{j\in [b]} \supp{\phi_j}$ is distributed as $\Bern{1/2}$.
Hence, $\ALG_{b-1}$ is a $(b-1)$-pass algorithm with memory at most $K'(8b(b-1)\log e)^{-1}$ that achieves regret $o(T^{\alpha}\cdot 4^{-(b-1)})$ over a time horizon $T-o(T)$, which is a contradiction! This completes the proof of our lower bound.

In the following section, we present our algorithmic results for this problem. Specifically, we design an algorithm that achieves a regret of $\widetilde{O}\left(T^{\frac{2^B}{2^{B+1}-1}}\sqrt{KB}\right)$ in $B$ passes given even just constant arm memory. Furthermore, our algorithm is able to achieve this regret, not just in expectation, but also with any polynomially high probability. This regret guarantee nearly matches the above lower bound, proving our above lower bound is nearly tight.

\section{Limited Memory Multi-Pass Algorithms for Streaming Bandits}
\label{sec:upper}
In this section, we present our worst-case and instance-dependent regret upper bounds for limited memory multi-pass streaming bandits. The following theorem characterizes our algorithmic results.

\begin{theorem}[$B$-Pass Upper Bound]
\label{thm:main}
Given a time horizon $T$, a stream of $K$ arms, and number of passes $1\leq B< \log\log T$, there exists a $B$-pass algorithm that uses $O(1)$ words of memory, and with probability $1-1/\poly(T)$, achieves cumulative regret 
\[R_T \leq O\left(T^{\frac{2^B}{2^{B+1}-1}}\sqrt{KB\log T} \right).\]
Furthermore, supposing the arms $\K$ had mean rewards $\{\mu^*_j\}_{j\in\K}$, then given number of passes $1\leq B < \log T$, there exists a $B$-pass algorithm that uses $O(1)$ words of memory, and with probability $1-1/\poly(T)$, achieves a cumulative regret 
\[R_T \leq O\left(\sum_{j\in \S} \frac{T^{1/(B+1)} \log T + B \log\left((\Delta^*_j)^2T/\log T\right)}{\Delta^*_j}\right),\]
where $\S\subset \K$ is the set of strictly sub-optimal arms in $\K$, and for any sub-optimal arm $j\in \S$, $\Delta^*_j := \max_{\{i\in \K\}}\mu^*_i - \mu^*_j$ is the regret due to playing arm $j$.
\end{theorem}
Note that no assumptions are made about the stream order, and that these regret guarantees hold even when the order of arms is allowed to change (potentially adversarially) across rounds. 

In the constant pass regime, the worst-case regret achievable matches our lower bound up to just a $\sqrt{K\log T}$ factor, implying our results are essentially tight for this regime. Our result further implies one can achieve a worst-case regret of $O(\sqrt{KT\log T\cdot \log\log T})$ in just $\log\log T$ passes over the stream, which matches the optimal regret achievable by even an unbounded memory algorithm up to a $\sqrt{\log T \log\log T}$ factor. With regards to instance-dependent regret, the picture is slightly different where we need $\log T$ passes (though still sublinear) over the stream to achieve regret $O(\sum_{j\in \S} (\Delta^*_j)^{-1} \log^2 T)$, which matches the instance-optimal regret achievable by an unbounded memory algorithm up to a $\log T$ factor.

Moreover, observe that our upper bound has no dependence on the available memory; the aforementioned regret guarantees can be achieved with even just constant memory. This upper bound together with the lower bound from \Sec{sec:lower} effectively demonstrates a \emph{sharp threshold} for $B$-pass regret as a function of  memory $M$:  
with $M=O(1)$ one can achieve $\widetilde\Theta\paren{T^{\half + \frac{1}{2^{B+2}-2}}}$ regret, and
increasing $M$ to any quantity that is $o(K/B^2)$ has 
almost no impact on further reducing this regret.
We now present our algorithm and its analysis.

\subsection{Algorithm}

Our proposed algorithm builds upon the classical Sequential-Elimination algorithm, where one maintains an ``active set'' of arms which are played in a round-robin manner until sufficient evidence is gathered indicating the sub-optimality of some arm, at which point it is permanently discarded from the active set. 

In our limited memory setting, it is not possible to have all arms in the active set as the number of statistics we can save at any given time is bounded by $M$. Therefore, our active set is of size roughly equal to our memory, and we play the least played arm $i_{\min}\in \M$ in our active set until we gather sufficient evidence to discard a sub-optimal arm, after which the next arm from the stream is read into our active set. This requires storing 2 statistics per arm $i\in \M$ in memory -- the cumulative reward observed $r_i$, as well as the number of times the arm was played $n_i$. For ease of exposition, we shall assume that both of these can be stored in a single word of memory. 

In addition to these arms in our active set, we reserve an additional word of memory to store the arm $\tilde{i}$ (and its statistics $\tilde{\ell}$) we have estimated to be the best. This arm $\tilde{i}$ serves two important purposes. Firstly, it is exploited until the end of the time horizon after we have exhausted our budget on the number of passes. Secondly, its stored statistic $\tilde{\ell}$ which is a lower bound on its estimated mean reward, is used to quickly identify and discard sub-optimal arms from memory. This is necessary in this limited memory streaming setting as unlike the full memory setting, it is not possible to permanently discard bad arms. Even after establishing the sub-optimality of bad arms, their identity is forgotten when they are discarded from memory, and will be repeatedly encountered in subsequent passes at which point it eliminating them without incurring too much regret becomes crucial. 

However, as established in our lower bound construction, the limited memory setting has an inherent risk associated: there can be some high value arm somewhere ahead in the stream that has not been read into memory yet, due to which overplaying the arms currently in memory can lead to large regret. As a result, we need to maintain a careful balance between playing arms in memory and exploring further into the stream. To address this problem, we borrow an idea from the limited-adaptivity framework for multi-armed bandits \cite{Perchet+15,gao2019batched}, where we additionally impose a cap on the maximum number of times any arm can be played in a single pass, effectively limiting the length of exploration to (roughly) $N^b$ in any single pass $b\in [B]$. If all arms in the active set have been played equal to the cap for that pass without any arm being discarded, then an arbitrary arm is ejected to make room for the next arm in the stream. This cap grows across passes, and in pass $b$ is roughly $T^{\frac{2^{B+1}(1-1/2^b)}{2^{B+1}-1}}$ if the objective is to minimize worst-case regret $(w=1)$, and $T^{\frac{b}{B+1}}$ if the objective is to minimize instance-dependent regret $(w=0)$. Intuitively, one can think of this as approximating the mean rewards of arms with an increasingly finer precision across passes. If a crude estimate of the reward suffices to discard a suboptimal arm, then it does so. Otherwise, this specific choice of the cap ensures that this arm has not been explored enough to incur significant regret. The following is a formal description of this algorithm.

\begin{algorithm}[ht]
\caption{Memory Bounded Successive Elimination}
\begin{algorithmic}
\label{alg:main}
\STATE{\textbf{Input.} Memory $M$; number of passes $B$; time horizon $T$; \\$\qquad\quad$variable $w=\begin{cases}1, \text{ if minimizing worst-case regret}\\0, \text{ if minimizing instance-dependent regret} \end{cases}$ }
\STATE{Let arms in memory $\mathcal{M}\gets \emptyset$, $N^0 \leftarrow 1$.
\STATE{Set aside a single word of memory: set (estimated) best arm $\tilde{i}\leftarrow \emptyset$, lower confidence bound $\tilde{\ell} \leftarrow 0$}}
\FOR{pass $b=1,\ldots,B$}
	\STATE {Set the maximum number of pulls across all arms in pass $b$,\\ $N^b \leftarrow \begin{cases} T^{\frac{2^B}{2^{B+1} - 1}} \sqrt{N^{b-1}}, &\text{ if } w=1 \text{ (minimize worst-case regret) }\\ T^{\frac{1}{B+1}}N^{b-1}, &\text{ if } w=0 \text{ (minimize instance-dependent regret)} \end{cases}$ }
	\STATE{For all arms $i\in \mathcal{M}$, set $n^b_i \leftarrow 0$, and $r^b_i\leftarrow 0$}
	\WHILE{pass is not finished}
		\WHILE{$|\mathcal{M}|< M-1$}
			\STATE{$\mathcal{M} \leftarrow \mathcal{M}\cup\{i\}$, where $i$ is the next arm in the stream that is not already in memory.}
			\STATE{Set number of pulls $n^b_i\leftarrow 0$, cumulative reward $r^b_i \leftarrow 0$}
		\ENDWHILE
		\STATE{Let $i_{\min} \leftarrow \argmin_{i\in \mathcal{M}} n^b_i$ be the least played arm in memory (ties broken arbitrarily)}
		\IF{$\left(n^b_{i_{\min}} \geq N^b/(KB) \text{ and } w=1\right)$ or $\left(n^b_{i_{\min}} \geq N^b \text{ and } w=0\right)$}
			\STATE{Discard an arbitrary arm $i\in\M$ from memory; $\mathcal{M} \leftarrow \mathcal{M}\setminus\{i\}$}
		\ELSE
			\STATE{Play arm $i_{\min}$ once, and observe reward $r$}
			\STATE{Update $r^b_{i_{\min}} \leftarrow r^b_{i_{\min}} + r$, and $n^b_{i_{\min}} \leftarrow n^b_{i_{\min}} + 1$}
		\ENDIF
		\STATE{Update $\tilde{\ell} \leftarrow \max_{i\in \mathcal{M}} r^b_i/n^b_i - \sqrt{(5\log T)/n^b_i}$; and $\tilde{i} \leftarrow \argmax_{i\in \mathcal{M}} r^b_i/n^b_i - \sqrt{(5\log T)/n^b_i}$}
		\IF{there exists an arm $j\in \mathcal{M}$ such that $r^b_j/n^b_j + \sqrt{(5\log T)/n^b_j} < \tilde{\ell}$}
			\STATE{Discard arm $j$ from memory; $\mathcal{M} \leftarrow \mathcal{M}\setminus\{j\}$}
		\ENDIF
	\ENDWHILE
\ENDFOR
\STATE{Play the estimated best arm $\tilde{i}$ until the end of the time horizon}
\end{algorithmic}
\end{algorithm}

It is clear that Algorithm~\ref{alg:main} uses memory at most $M$, and performs $B$ passes over the stream given any input parameters $M,B$. Furthermore, observe that while we allow for larger memory, our algorithm just needs $M=2$ words of memory. We shall now analyze the regret guarantees of Algorithm~\ref{alg:main}, which are restated here for convenience.  

\subsection{Analysis for Worst-Case Regret $(w=1)$}
As mentioned earlier, our algorithm cleverly balances playing the arms currently in memory, thereby gathering valuable information about them, and quickly exploring ahead into the stream to find potential high value arms by setting a cap $(\approx N^b)$ on the maximum number of times  any arm $i$ can be played in any pass $b$. This cap is raised across passes in a systematic way, with the choice of growth rate guaranteeing that the total regret incurred due to playing suboptimal arms will be small. 

At a high level, our proof shows that if across all passes $b\in [B]$, the observed mean rewards $r^b_j/n^b_j$ for all arms $j\in \K$ are not too far from their true mean rewards $\mu^*_j$, then the estimated best arm $\tilde{i}$ at the end of any pass $b$ is a good proxy for the true best arm $i^*$ for that pass. Specifically, the mean reward of $\tilde{i}$ closer to the mean reward of $i^*$ than the precision $(\approx \sqrt{1/N^b})$ with which we estimate the means in that pass. This can be used to eliminate any bad arms that are distinguishable from the best arm in pass $b$ (based on the precision set for that pass), but only starting the following pass $b+1$. Due the delayed nature of this information, i.e. the estimated best arm becomes ``good enough'' for elimination purposes in a pass $b$ only after the true best arm, (or a good proxy for it) are encountered in the stream in that pass $b$. Therefore, suboptimal arms that transitioned from being indistinguishable in pass $b-1$ to distinguishable in pass $b$, and appeared early on in the stream can potentially be overplayed because the estimated best arm has not been updated yet. This would incur more regret than is desirable, but only up to a multiplicative $T^{2^B/(2^{B+1}-1)}$ factor (matching our $B$-pass lower bound) due to our choice $(N^b)$ of the cap on the number of times an arm can be played in any pass. This gives our final guarantee on the upper bound on the regret of our algorithm.  

We shall now formally prove this bound on the worst-case regret achieved by our algorithm. We begin with the following simple lemma, which bounds the deviation in the observed mean rewards of any arm from its true mean reward. Specifically, this lemma says that whenever an arm is stored in memory, its true mean reward will lie within a confidence ball of radius $O(\sqrt{\log T/n})$ if it has been played $n$ times since being read into memory. Furthermore, this property would hold for all arms across all rounds with a polynomially large probability.

\begin{lemma}
\label{lemm:event_g}
Let $\K$ be the set of arms in the stream, where arm $i\in \K$ has mean reward $\mu^*_i$, and let $B$ be the total number of passes. Whenever arm an $i\in \K$ is present in memory in pass $b\in [B]$ of the stream, we define the event
\[\mathcal{E}_{i,b} := \left|\mu^*_i - \frac{r^b_i}{n^b_i}\right| < \sqrt{\frac{c\log T}{n^b_i}},\]
where $r^b_i$ represents the observed cumulative reward of arm $i$ in the $b^{th}$ pass, $n^b_i$ represents the number of times arm $i$ was played in the $b^{th}$ pass, and $c\geq 5$ is any constant. Then we have that the event $\mathcal{E}:=\cap_{i\in \K,b\in [B]} \mathcal{E}_{i,b}$ occurs with probability at least $1-2/T^{c-4}$.
\end{lemma}

\begin{proof}
Consider any fixed arm $i\in\K$. We shall assume that the rewards for this arm are sampled from the corresponding reward distribution, and written on a tape (of length at most $T$). Whenever the algorithm chooses to play this arm $i$ in some round, it simply reads the realized reward from the next cell on the tape. By definition, as long as an arm is in memory, the algorithm maintains a running average of the observed rewards of that arm, resetting the running average every time it starts a new pass or loads the arm into memory, treating the last cell on the tape as a new starting point for counting rewards for this arm. For a fixed starting point on the tape after which it started keeping count of the rewards for the arm, the probability that $|\mu^*_i - r_i/n_i| \geq \sqrt{c\log T/n_i}$ for a fixed value of $n_i$ is at most $2/T^c$ by Hoeffding's inequality. Therefore, the probability that this event occurs for some starting point on the tape and some value of $n_i$, by a union bound, is at most $2/T^{c-2}$. Observe that this event is exactly $\neg \mathcal{E}_{i,b}$ for some fixed pass $b$, as the running average is reset (either when the arm was in memory at the start of the pass, or was first loaded into memory at some point during that pass) at most once during that pass. Therefore, for any fixed arm $i\in \K$, the probability that event $\neg \mathcal{E}_{i,b}$ occurs for some pass $b\in [B]$, by a union bound over the passes, is at most $2B/T^{c-2}$. Finally, taking a union bound over all arms, the probability that event $\neg \mathcal{E}_{i,b}$ occurs for some pass arm $i\in \K$, and some pass $b\in [B]$ is at most $2KB/T^{c-2}$. This gives us our claimed bound, since $K,B\leq T$. 
\end{proof}

For simplicity, we assume $c=5$ in Algorithm~\ref{alg:main} and in the subsequent proof, which gives us that the above defined ``good event'' of interest occurs with probability at least $1-2/T$. Henceforth, we shall assume that this event occurs, following which our regret guarantees hold with probability $1$.

\begin{proof} 
(of \Thm{thm:main} (worst-case upper bound))
Let $\mathcal{E} := \cap_{i\in \K,b\in [B]} \mathcal{E}_{i,b}$ be the good event of interest defined in Lemma~\ref{lemm:event_g}. Then we shall prove that conditioned on event $\event$, the cumulative regret $R_T$ of Algorithm~\ref{alg:main} set for worst-case regret minimization $(w=1)$ is 
\[R_T \leq O\left(T^{\frac{2^B}{2^{B+1}-1}}\sqrt{KB\log T}\right),\] 
with probability $1$. Prior to formally proving this, observe that this also automatically implies a $\Exp[R_T] \leq O\left(T^{\frac{2^B}{2^{B+1}-1}}\sqrt{KB\log T}\right)$ result for the expected regret of our algorithm as 
\begin{align*}
\Exp[R_T] &= \Pr(\mathcal{E}) \Exp[R_T\mid\mathcal{E}] + (1-\Pr( \mathcal{E})) \Exp[R_T\mid \neg \mathcal{E}]\\
&\leq \left(1-\frac{1}{\poly (T)}\right)\Exp(R_T|\event) + \frac{1}{\poly (T)} \cdot T\\
& \leq O\left(T^{\frac{2^B}{2^{B+1}-1}}\sqrt{KB\log T}\right)
\end{align*}

Let $\mu^*_{\max} := \max_{i\in \K} \mu^*_i$ be the largest expected reward of any arm in the stream, and let $\mathcal{S} :=\{j\in \K:~\mu^*_j<\mu^*_{\max}\}$ be the set of all suboptimal arms, with $ \Delta_j:=\mu^*_{\max} - \mu^*_j$ being the regret due to playing any suboptimal arm $j\in \mathcal{S}$. Furthermore, let $i^*\in \K\setminus \S$ be any arbitrary optimal arm, which we shall henceforth refer to as the best arm.

For any pass $b\in [B]$ and any arm $j\in \K$, let $m^b_j$ be the maximum number of times arm $j$ was played in pass $b$. Furthermore, let $R^b_j = m^b_j\Delta_j$ be the regret incurred by the algorithm by playing arm $j$ in the $b^{th}$ pass, and subsequently, let $R^b = \sum_{j\in \mathcal{S}} R^b_j$ be the total regret incurred in the $b^{th}$ pass. Finally, let $R_{\tilde{i}}$ be the regret incurred by playing the estimated best arm $\tilde{i}$ at the end of the $B^{th}$ pass until the end of the time horizon. Therefore, we have that
\begin{align*}
R_T &= \sum_{b\in [B]}R^b +  R_{\tilde{i}} = \sum_{j\in\S}R^1_j + \sum_{b=2}^B\sum_{j\in \mathcal{S}}R^b_j + R_{\tilde{i}} \leq \frac{T^{\frac{2^B}{2^{B+1} - 1}}}{B} + \sum_{b=2}^B\sum_{j\in \mathcal{S}}R^b_j + R_{\tilde{i}},
\end{align*}
where the final inequality follows by observing that the maximum number of times any arm is played in the first pass is at most $T^{\frac{2^B}{2^{B+1} - 1}}/(KB)$. 

We shall now present the key implication of event $\event$, which basically guarantees that the best arm $i^*$ will necessarily be played the maximum allowable number of times in every pass, i.e. $i^*$ cannot be prematurely discarded from memory after it has been read in the stream. This gives us certain desirable guarantees about the true mean reward of the estimated best arm saved in $\tilde{i}$, as well as the maximum number of times $m^b_j$ any suboptimal arm $j\in \S$ can be played in any pass $b\geq 2$, which will be crucial in bounding the cumulative regret of the algorithm.   

\begin{claim}
\label{clm:event_e}
Given that event $\event$ occurs, arm $i^*$ will necessarily be played $N^b/(KB)$ times in every pass $b\in [B]$. Consequently, for any pass $b\geq 2$, we have 
\[\tilde{\ell} \geq \mu^*_{\max} - 2\sqrt{(5KB\log T)/N^{b-1}},\]
at all times in pass $b$. 
\end{claim}
\begin{proof}
This follows by observing that $\tilde{\ell}$ can never exceed $\mu^*_{\max}$ at any point during the execution of the algorithm. This is because for any arm $j\in \K$, event $\event$ guarantees that its observed mean reward $r^b_j/n^b_j < \mu^*_j + \sqrt{(5\log T)/n^b_j}$ in every pass $b\in [B]$, which guarantees $r^b_j/n^b_j - \sqrt{(5\log T)/n^b_j} < \mu^*_j \leq \mu^*_{\max}$. Furthermore, for the best arm $i^*$, we have that $r^b_{i^*}/n^b_{i^*} > \mu^*_{\max} - \sqrt{(5\log T)/n^b_{i^*}}$ in every pass $b\in [B]$, which guarantees $r^b_{i^*}/n^b_{i^*} + \sqrt{(5\log T)/n^b_{i^*}} > \mu^*_{\max}$. Therefore, the only way the best arm can be discarded from memory in any pass is if the memory was full with all arms in memory having been played $N^b/(KB)$ times without being eliminated. 

To see why this implies a lower bound on the value of $\tilde{\ell}$ in any pass $b\geq 2$, observe that the best arm $i^*$ was played $m^{b-1}_{i^*} = N^{b-1}/(KB)$ times in the previous pass, implying that the observed mean reward of the best arm $i^*$ at the end of pass $b-1$ was $r^{b-1}_{i^*}/m^{b-1}_{i^*}\geq \mu^*_{\max} - \sqrt{(5KB\log T)/N^{b-1}}$. Therefore, the value of $\tilde{\ell}$ at the end of pass $b-1$ is at least $\mu^*_{\max} - 2\sqrt{(5KB\log T)/N^{b-1}}$, with the claim following from the fact that $\tilde{\ell}$ is a strictly increasing value. 
\end{proof}

We are now ready to bound the cumulative regret $R^b_j$ due to playing any suboptimal arm $j\in \S$ in any pass $b\geq 2$. Let $m^b_j$ be the final time arm $j$ was played in round $b$. Since suboptimal arm $j$ was played that one last time, it must be the case that 
\begin{equation}
\label{eqn:event_j_1}
r^b_j/(m^b_j -1) + \sqrt{(5\log T)/(m^b_j-1)} \geq \tilde{\ell},
\end{equation}
where $\tilde{\ell}$ was the value of the largest lower confidence bound at that moment. Furthermore, by definition of event $\mathcal{E}$ and \Clm{clm:event_e}, it must be that 
\begin{equation}
\label{eqn:event_j_2}
r^b_j/(m^b_j -1) < \mu^*_j + \sqrt{(5\log T)/(m^b_j-1)},  \text{ and } \tilde{\ell} \geq \mu^*_{\max} - 2\sqrt{(5KB\log T)/N^{b-1}}.
\end{equation}

Substituting these bounds into Equation~\ref{eqn:event_j_1}, we get 
\[\Delta_j < 2\sqrt{(5\log T)/(m^b_j-1)} + 2\sqrt{(5KB\log T)/N^{b-1}},\]
and therefore, the cumulative regret due to playing arm $j\in \S$ in pass $b\geq 2$ is at most
\begin{align*}
R^b_j &= m^b_j\Delta_j\\
&< 2m^b_j\sqrt{(5\log T)/(m^b_j-1)} + 2m^b_j\sqrt{(5KB\log T)/N^{b-1}}\\
&< 2\sqrt{6m^b_j\log T} + 2m^b_j\sqrt{(5KB\log T)/N^{b-1}},
\end{align*}

Therefore, the total regret in pass $b\geq 2$ is given by
\begin{align*}
R^b &= \sum_{j\in \mathcal{S}} \Exp(R^b_j|\mathcal{E})\\
 &\leq 2\sum_{j\in\mathcal{S}}\sqrt{6m^b_j\log T} + \sum_{j\in\mathcal{S}}2m^b_j\sqrt{\frac{5KB\log T}{N^{b-1}}}\\
 &\leq 2\sum_{j\in\mathcal{S}}\sqrt{6m^b_j\log T} + \frac{2N^b}{B}\sqrt{\frac{5KB\log T}{N^{b-1}}}\\
 &\leq 2\sum_{j\in\mathcal{S}}\sqrt{6m^b_j\log T} + 2T^{\frac{2^B}{2^{B+1}-1}}\sqrt{\frac{5K\log T}{B}},
\end{align*}
where the penultimate inequality follows due to the fact that in any pass $b\in [B]$, $\sum_{j\in \mathcal{S}} m^b_j \leq |\mathcal{S}| N^b/(KB) \leq N^b/B$, and the final inequality follows due to the fact that $N^b = T^{2^B/(2^{B+1}-1)}\sqrt{N^{b-1}}$. Therefore, the cumulative regret in the first $B$ passes is
\begin{align*}
\sum_{b=2}^B R^b &\leq 2\sum_{b=2}^B\sum_{j\in\mathcal{S}}\sqrt{6m^b_j\log T} + \sum_{b=2}^B 2T^{\frac{2^B}{2^{B+1}-1}}\sqrt{\frac{5K\log T}{B}}\\
&\leq 4\sqrt{6\log T} \sum_{b=2}^B\sum_{j\in\mathcal{S}}\sqrt{m^b_j} + 2T^{\frac{2^B}{2^{B+1}-1}}\sqrt{5KB\log T}.
\end{align*}
Due to the nature of the total number of pulls $N^b$ in any pass $b\in [B]$, we have $\sum_{b =2}^B\sum_{j\in\mathcal{S}} m^b_j \leq 2N^B/B$. Applying Jensen's inequality to the concave function $f(x)=\sqrt{x}$, we have 
\[\frac{1}{|\mathcal{S}|B}\sum_{b=2}^B\sum_{j\in\mathcal{S}}\sqrt{m^b_j} \leq \sqrt{\frac{1}{|\mathcal{S}|B}\sum_{b=2}^B\sum_{j\in\mathcal{S}}m^b_j} \leq \frac{1}{B}\sqrt{\frac{2N^B}{|\mathcal{S}|}},\] 
giving us 
\[\sum_{b=2}^B\sum_{j\in\mathcal{S}}\sqrt{m^b_j} \leq \sqrt{2|\mathcal{S}|N^B} \leq \sqrt{2KN^B}.\]
Now observe that $N^B = T^{1-\frac{1}{2^{B+1}-1}}$, giving us our final bound on the cumulative regret of the algorithm in the first $B$ passes as
\[\sum_{b=2}^BR^b \leq O\left(T^{\frac{2^B}{2^{B+1}-1}}\sqrt{KB\log T}\right). \]

We shall finally bound the cumulative regret $R_{\tilde{i}}$ due to playing the estimated best arm $\tilde{i}$ until the end of the time horizon. Since this estimated best arm $\tilde{i}$ was responsible for setting the final value of $\tilde{\ell}$ at the end of the $B^{th}$ pass, it must be the case that $\mu^*_{\tilde{i}} > \tilde{\ell} > \mu^*_{\max} - 2\sqrt{(5KB\log T)/N^B}$, with the final inequality following due to \Clm{clm:event_e}, which guarantees that 
\[\Delta_{\tilde{i}} < 2\sqrt{(5KB\log T)/N^B}.\]
Furthermore, arm ${\tilde{i}}$ can be played at most $T$ times after pass $B$ until the end of the time horizon, giving us that the regret due to playing arm $j$ 
\[R_{\tilde{i}} < T\Delta_{\tilde{i}} < 2T\sqrt{(5KB\log T)/N^B} = T^{\frac{2^B}{2^{B+1}-1}}\sqrt{5KB\log T},\]
where the final inequality follows by observing $N^B = T^{1-\frac{1}{2^{B+1} - 1}}$, and therefore, $\frac{T}{\sqrt{N^B}} = T^{\frac{1}{2} + \frac{1}{2(2^{B+1} - 1)}} = T^{\frac{2^B}{2^{B+1}-1}}$. Combining these bounds on $R^1, \sum_{b=2}^B R^b$, and $R_{\tilde{i}}$, we get our claimed bound on the cumulative regret achieved by our algorithm. This concludes the proof of \Thm{thm:main} for worst-case regret.

\end{proof}
We shall now analyze the regret of Algorithm~\ref{alg:main} when it is set to minimizing instance-dependent regret $(w=0)$.

\subsection{Analysis  for Instance-Dependent Regret $(w=0)$}
The analysis of the instance-dependent regret of Algorithm~\ref{alg:main} is conceptually identical to the previous analysis. It is based on this same intuition that the estimated best arm is a good enough proxy to eliminate distinguishable bad arms from memory, but there is this delay in information due to the best arm (or a proxy for the best arm) can be encountered very late in the stream in some pass, causing some bad arms that just transitioned from being indistinguishable in the previous pass to being distinguishable in the current pass, to be potentially overplayed but only up to a multiplicative $T^{1/(B+1)}$ factor. The additional additional $B$ factor comes from the fact that the identities and statistics of discarded arms is forgotten, due to which the suboptimality of even distinguishable arms has to be repeatedly established in every subsequent pass. Even though the suboptimality of any distinguishable arm can be established in a future pass while incurring the optimal regret in that pass, this process has to be repeated in every pass until $B$. We now formally prove the claimed instance-dependent regret guarantee.

\begin{proof}(of \Thm{thm:main} (instance-dependent upper bound))
Let $\mathcal{E} := \cap_{i\in \K,b\in [B]} \mathcal{E}_{i,b}$ be the good event of interest defined in Lemma~\ref{lemm:event_g}. We shall prove that conditioned on event $\event$, the cumulative regret of Algorithm~\ref{alg:main} set for instance-dependent regret minimization $(w=0)$ is 
\[R_T \leq O\left(\sum_{j\in \S} \frac{T^{1/(B+1)} \log T + B \log\left((\Delta^*_j)^2T/\log T\right)}{\Delta^*_j}\right),\] 
with probability $1$. We note that since event $\event$ occurs with a polynomially large probability as proved in \Lem{lemm:event_g}, this also implies the same bound on the expected regret of the aforementioned algorithm. 

Let $\mu^*_{\max} := \max_{i\in \K} \mu^*_i$ be the largest expected reward of any arm in the stream, and let $\mathcal{S} :=\{j\in \K:~\mu^*_j<\mu^*_{\max}\}$ be the set of all suboptimal arms, with $ \Delta_j:=\mu^*_{\max} - \mu^*_j$ being the regret due to playing a suboptimal arm $j\in \mathcal{S}$. Furthermore, let $i^*\in \K\setminus \S$ be any arbitrary optimal arm, which we shall henceforth refer to as the best arm.

For any pass $b\in [B]$ and any arm $j\in \K$, let $m^b_j$ be the maximum number of times arm $j$ was played in pass $b$. Furthermore, let $R^b_j = m^b_j\Delta_j$ be the regret incurred by the algorithm by playing arm $j$ in the $b^{th}$ pass, and subsequently, let $R_j = \sum_{b\in [B]} R^b_j$ be the total regret due to playing a suboptimal arm $j\in\S$. Finally, let $R_{\tilde{i}}$ be the regret incurred by playing the estimated best arm  $\tilde{i}$ at the end of the $B^{th}$ pass until the end of the time horizon. Therefore, we have that
\begin{align*}
R_T &= \sum_{j\in \S}R_j +  R_{\tilde{i}} = \sum_{j\in \mathcal{S}}\sum_{b\in [B]}R^b_j + R_{\tilde{i}}.
\end{align*}

As before, we shall use a crucial implication of event $\event$ to bound the instance-dependent regret of our algorithm. The following claim is the analog of \Clm{clm:event_e} in this setting, with a minor difference due to the fact that an arm is played a maximum of $N^b$ times in epoch $b$ when minimizing instance-dependent regret $(w=0)$ as compared to $N^b/(KB)$ when minimizing worst-case regret $(w=1)$.

\begin{claim}
\label{clm:event_e_instance}
Given that event $\event$ occurs, arm $i^*$ will necessarily be played $N^b/(KB)$ times in every pass $b\in [B]$. Consequently, for any pass $b\geq 2$, we have 
\[\tilde{\ell} \geq \mu^*_{\max} - 2\sqrt{(5\log T)/N^{b-1}},\]
at all times in pass $b$. 
\end{claim}

The proof of this claim follows identically to that of \Clm{clm:event_e}. For any suboptimal arm $j\in \mathcal{S}$, we define the \emph{distinguishing pass} $b_j$ to be the smallest value of $b\in [B]$ such that $\Delta^*_j > 4\sqrt{(5\log T)/T^{{b}/{(B+1)}}}$. Intuitively, this represents the pass in which the precision to which we estimate the gap parameters exceeds the value of $\Delta^*_j$, due to which it becomes possible to efficiently infer the sub-optimality of arm $j$. We now claim that in any pass $b>b_j$, arm $j$ will be played $m^b_j \leq 80\log T/(\Delta^*_j)^2$ times in that pass. To see this, observe that in any pass $b>b_j$ we have that 
\begin{align*}
\tilde{\ell} &\geq \mu^*_{\max} - 2\sqrt{(5\log T)/N^{b-1}} \\
&\geq \mu^*_{\max} - 2\sqrt{(5\log T)/N^{b_j}} \\
&> \mu^*_{\max} - \Delta^*_j/2\\
& = (\mu^*_{\max} + \mu^*_j)/2.
\end{align*} 
Now in pass any pass $b>b_j$, event $\mathcal{E}$ further guarantees that after any $n^b_j$ pulls of arm $j$, we will have 
\begin{align*}
r^b_j/n^b_j + \sqrt{(5\log T)/n^b_j} &< \mu^*_j + 2\sqrt{(5\log T)/n^b_j} \\
&\overset{(a)}{<} \mu^*_j + \Delta^*_j/2\\
&= (\mu^*_{\max} + \mu^*_j)/2,
\end{align*}
where equation $(a)$ follows by supposing arm $j$ was actually played $80\log T/(\Delta^*_j)^2$ times in that pass. This would guarantee that arm $j$ will be discarded from memory after $80\log T/(\Delta^*_j)^2$ pulls. Therefore, we have that the number of times arm $j$ is played in pass $b\in [B]$ is bounded as $m^b_j \leq N^b$ for $b\leq b_j$ and $m^b_j \leq 80\log T/(\Delta^*_j)^2$ for $b>b_j$, giving us the total regret due to playing arm $j$ as
\begin{align*}
R_j &= \sum_{b\in [B]} R^b_j\\
&=\sum_{b\in [B]} m^b_j\Delta^*_j\\
&\leq \sum_{b\leq b_j} N^b\Delta^*_j + \sum_{b>b_j} \frac{80\log T}{(\Delta^*_j)^2} \Delta^*_j\\
&\leq \Delta^*_j \sum_{b=1}^{b_j} T^{b/(B+1)} + (B-b_j-1)\frac{80\log T}{\Delta^*_j} \\
&\leq \Delta^*_j \left(\frac{T^{1/(B+1)}(T^{b_j/(B+1)}-1)}{T^{1/(B+1)} - 1}\right) + (B-b_j-1)\frac{80\log T}{\Delta^*_j} \\
& = \Delta^*_j \left(\left(1+\frac{1}{T^{1/(B+1)} -1}\right)\left(T^{b_j/(B+1)}-1\right)\right) + (B-b_j-1)\frac{80\log T}{\Delta^*_j}
\end{align*}

Observe that $b_j$ is the smallest value of $b\in [B]$ such that $\Delta^*_j > 4\sqrt{(5\log T)/T^{{b}/{(B+1)}}}$, implying $\Delta^*_j \leq 4\sqrt{(5\log T)/T^{{(b_j-1)}/{(B+1)}}}$, giving us 
\[T^{b_j/(B+1)} \leq \frac{80}{(\Delta^*_j)^2}T^{1/(B+1)}\log T, \text{ and } b_j > \frac{(B+1)}{\log T}\log\left(\frac{80\log T}{(\Delta^*_j)^2}\right). \]

Therefore, substituting these values into the above equation, and using the fact that $(B+1)\leq \log T$, we get
\[R_j \leq O\left(\frac{T^{1/(B+1)} \log T + B \log\left((\Delta^*_j)^2T/\log T\right)}{\Delta^*_j}\right),\]
and therefore,
\[\sum_{j\in \S}R_j \leq O\left(\sum_{j\in \mathcal{S}} \frac{T^{1/(B+1)} \log T + B \log\left((\Delta^*_j)^2T/\log T\right)}{\Delta^*_j}\right)\]

To bound $R_{\tilde{i}}$, observe that for arm $\tilde{i}$, it must have been the case that $\Delta^*_{\tilde{i}} < 4\sqrt{(5\log T)/N^B}$ due to \Clm{clm:event_e_instance}.
Therefore, the regret due to playing any arm $\tilde{i}$ until the end of the time horizon can be bounded by
\[R_{\tilde{i}} \leq T\Delta^*_{\tilde{i}} \leq \frac{T}{\Delta^*_{\tilde{i}}}(\Delta^*_{\tilde{i}})^2 \leq \frac{80T^{1/(B+1)} \log T}{\Delta^*_{\tilde{i}}},\]

where the final inequality follows from the fact that $N^B = T^{\frac{B}{B+1}}$. Combining these two bounds gives us our claimed upper bound on the cumulative regret as
\[R_T \leq O\left(\sum_{j\in \mathcal{S}} \frac{T^{1/(B+1)} \log T + B \log\left((\Delta^*_j)^2T/\log T\right)}{\Delta^*_j}\right)\]

This completes the proof of our upper bound result.
\end{proof}

\section{Discussion and Conclusion}
\label{sec:conc}
We studied the stochastic $K$-armed bandits problem in a limited memory, multi-pass streaming setting, where we study the interplay between the available memory $M$, the number of passes $B$, and the regret $R_T$ over a time horizon $T$. We showed that any $B$-pass algorithm with memory $o(K/B^2)$ must necessarily incur $\Omega\left(4^{-B}T^{\half + \frac{1}{2^{B+2}-2}}\right)$ regret in expectation. Moreover, we showed that it is possible to achieve $\widetilde{O}(T^{\half + \frac{1}{2^{B+2}-2}}\sqrt{KB})$ regret with any polynomially large probability given $B$ passes and just $O(1)$ memory. 
These results uncover a surprising phenomenon: 
increasing the memory beyond $O(1)$ memory to any quantity 
that is $o(K/B^2)$ has almost no effect on reducing the expected worst-case regret. 

Our work highlights  some interesting directions for future work. First, while our results are essentially tight for constant-pass algorithms, there is a gap of $1/2^B$ between our upper and lower bound on the regret when $B$ is a superconstant. Second, it might also be worth exploring the regret landscape in the memory range of $\Omega(K/B^2)$ to $K-1$, for superconstant $B$. Finally, what is the best instance-dependent regret one can achieve in this limited-memory multi-pass streaming setting? Our work establishes an instance-dependent regret upper bound of $\widetilde{O}( (T^{1/(B+1)}+B)\sum_{i\in \S}1/\Delta^*_i)$, but leaves open the question of a matching lower bound.

\section*{Acknowledgements}
This work was supported in part by NSF awards CCF-1763514, CCF-1934876, and CCF-2008305.

\bibliographystyle{acm}
\bibliography{streaming_bandits}

\appendix

\section{Information-Theoretic Preliminaries}
\label{sec:facts}
In this section, we record some basic facts about entropy and mutual information that are used in in this paper. The proofs can be found 
in~\cite{ITbook}, Chapter~2. We also prove two crucial lemmas in this section, the first which highlights the difficulty of narrowing down the realization of a high entropy random variable to a small set of possibilities, and the second which bounds the parameter of a high entropy Bernoulli random variable.

\begin{fact}\label{fact:it-facts}
  Let $\rA$, $\rB$, and $\rC$ be three (possibly correlated) random variables.
   \begin{enumerate}
  \item \label{part:uniform} $0 \leq \HH(\rA) \leq \log{\card{\rA}}$, and $\en{\rA} = \log{\card{\rA}}$
    iff $\rA$ is uniformly distributed over its support.
  \item \label{part:info-zero} $\mi{\rA}{\rB \mid  \rC} \geq 0$. The equality holds iff $\rA$ and
    $\rB$ are \emph{independent} conditioned on $\rC$.
  \item \label{part:cond-reduce} Conditioning can only drop the entropy: $\HH(\rA \mid  \rB, \rC) \leq \HH(\rA \mid \rB)$.  The equality holds iff $\rA \perp \rC \mid \rB$.
  \item \label{part:chain-rule} Chain rule of mutual information: $\mi{\rA, \rB}{\rC} = \mi{\rA}{\rC} + \mi{\rB}{\rC \mid  \rA}$.
   \end{enumerate}
\end{fact}

For two distributions $\phi$ and $\psi$ over the same probability space, the \emph{Kullback-Leibler divergence} between $\phi$ and $\psi$ is defined as $\DD{\phi}{\psi}:= \Ex_{\rv{A} \sim \phi}\Bracket{\log\frac{\Pr_{\phi}(\rv{A})}{\Pr_{\psi}(\rv{A})}}$.
For our proofs, we need the following relation between mutual information and KL-divergence. 
\begin{fact}\label{fact:kl-info}
	For random variables $\rA,\rB,\rC$, 
	\[\mi{\rA}{\rB \mid \rC} = \Ex_{(b,c) \sim \distribution{\rB,\rC}}\Bracket{ \DD{\distribution{\rA \mid \rC=c}}{\distribution{\rA  \mid \rB=b,\rC=c}}} .\] 
\end{fact}

The following fact can be proven by bounding the KL-divergence by $\chi^2$-distance (see, e.g.,~\cite{GibbsE02}, Theorem~5). 
\begin{fact}\label{fact:kl-chi}
	For any two parameters $0 < p,q < 1$, 
	\begin{align*}
		\DD{\Bern{p}}{\Bern{q}} \leq \frac{\paren{p-q}^2}{q \cdot (1-q)} 
	\end{align*}
\end{fact}

The following lemma outlines the difficulty in 
narrowing down the realization of a high-entropy random variable to a small number of possibilities.

\begin{lemma}
\label{lem:bounded_mass}
Let $\rv{A}$ be a random variable supported over a set of size $A$ with entropy $\en{\rv{A}} \geq \log A - \gamma$ for some $\gamma\geq 0$. Then for any set $\S$ of size $|\S| = \beta A$ for any $\beta <1$, we have
 \[
\Pr(\rv{A} \in \S) \leq \log\left(1+ \beta\right) + \gamma
	\,.
 \]
 \end{lemma}
 \begin{proof}
Suppose for the sake of contradiction,
 there exists a set $\S$ of size $|\S| = 
 \beta A$ for some $\beta< 1$ such that $\Pr(\rv{A}\in \S) = \sum_{i\in \S} \Pr(\rv{A}=i) = \gamma' > \log (1+\beta) + \gamma$. Let $p_i = \Pr(\rv{A}=i)$.
Then we have that:
\begin{align*}
\en{\rv{A}} &= \sum_{i\in \A} p_i\log\frac{1}{p_i}\\
		& = \sum_{i\in \S} p_i\log\frac{1}{p_i} + \sum_{i\notin \S} p_i\log\frac{1}{p_i}\\
		& = \gamma' \sum_{i\in \S} \frac{p_i}{\gamma'}\log\frac{1}{p_i} + (1-\gamma')\sum_{i\notin  \S} \frac{p_i}{(1-\gamma')}\log\frac{1}{p_i}\\
		&\overset{(a)}{\leq} \gamma'\log \left(\sum_{i\in \S} \frac{p_i}{\gamma'} \cdot \frac{1}{p_i}\right) + (1-\gamma')\log \left(\sum_{i\notin \S} \frac{p_i}{(1-\gamma')} \cdot \frac{1}{p_i}\right)\\
		& = \gamma' \log \frac{\beta A}{\gamma'} + (1-\gamma') \log\frac{(1-\beta)A}{(1-\gamma')}   \\
		& = \log A -\gamma'  \underbrace{- \gamma'\log \frac{\gamma'}{2\beta} + (1-\gamma')\log \frac{1-\beta}{1-\gamma'}}_{f(\gamma')}
		\,,
\end{align*}
where equation (a) follows by the Jensen's inequality, as the two summations are expectations over the concave $\log$ function over the set $\S$ and the set $\supp{\rv{A}}\setminus \S$, respectively. One can verify that the function $f(\gamma')$ is concave, and is maximized at $\gamma' = 2\beta/(1+\beta)$ achieving a value of $\log(1+\beta)$. Therefore, we have that $-\gamma' + f(\gamma') \leq -\gamma' + \max_{\gamma'} f(\gamma') < -\gamma$ by choice of $\gamma'>\log(1+\beta) + \gamma$, giving us
\[\en{\rv{A}} <\log A - \gamma,\]
contradicting our initial assumption about the entropy of $\rv{A}$.
 \end{proof}
 
 Lastly, following lemma bounds the parameter of a high entropy Bernoulli random variable.
 \begin{lemma}
 \label{lem:bern_en}
 Given a Bernoulli random variable $\bY\sim \B(p)$ with entropy $\en{\bY} \geq 1 - \gamma$ for any $\gamma\leq 1/4$, then we have that
 \[\left|p-\half\right| \leq \sqrt{\frac{5\gamma \ln 4}{16}}\]
 \end{lemma}
 \begin{proof}
 Suppose for the sake of contradiction, there exists a parameter $p := 1/2 + \Delta$ such that $\Delta > \sqrt{\frac{5\gamma \ln 4}{16}}$, and for $\bY\sim \B(p)$, the entropy $\en{\bY} \geq 1 - \gamma$. Then we have
 \begin{align*}
 \en{\bY} &\leq (4p(1-p))^{1/\ln 4}\\
 &= (4(1/2 - \Delta)(1/2 + \Delta))^{1/\ln 4}\\
 &= (1-4\Delta^2)^{1/\ln 4}\\
 &\leq \exp\left(-4\Delta^2/\ln 4\right)\\
 &< \exp(-5\gamma/4)\\
 &< 1-\frac{5\gamma}{4} + \frac{25\gamma^2}{32}\\
 &< 1-\gamma, 
 \end{align*}
where the final inequality follows by the fact that $\gamma\leq 1/4$, and thus $25\gamma^2/32 < \gamma/4$. This contradicts the assumption that $\en{\bY} \geq 1 - \gamma$.
 \end{proof}

\end{document}

%% file: macros.tex
\newcommand{\rv}[1]{\ensuremath{\mathsf{#1}}}

\newcommand{\en}[1]{\ensuremath{\mathbb{H}(#1)}}
\newcommand{\mi}[2]{\ensuremath{\mathbb{I}(#1 \,; #2)}}

\newcount\Comments  
\newcommand{\kibitz}[2]{\ifnum\Comments=1{\color{#1}{#2}}\fi}

\newcommand{\Tuni}{\ensuremath{S_\textnormal{ui}}}

\newcommand{\distD}{\ensuremath{\dist^{\Delta,\psi}}}

\newcommand{\bI}{\ensuremath{\rv{I}}}
\newcommand{\bY}{\ensuremath{\rv{Y}}}

\newcommand{\rA}{\ensuremath{\rv{A}}}
\newcommand{\rB}{\ensuremath{\rv{B}}}
\newcommand{\rC}{\ensuremath{\rv{C}}}
\renewcommand{\S}{\ensuremath{\mathcal{S}}}
\newcommand{\M}{\ensuremath{\mathcal{M}}}
\newcommand{\A}{\ensuremath{\mathcal{A}}}
\newcommand{\T}{\ensuremath{\mathcal{T}}}
\newcommand{\K}{\ensuremath{\mathcal{K}}}
\newcommand{\R}{\ensuremath{\mathcal{R}}}

\newcommand{\dist}{\ensuremath{\mathcal{D}}}

\newcommand{\HH}{\ensuremath{\mathbb{H}}}

\newcommand{\DD}[2]{\ensuremath{\mathbb{D}(#1~||~#2)}}
\newcommand{\distribution}[1]{\ensuremath{\textnormal{dist}(#1)}}
\newcommand{\supp}[1]{\ensuremath{\textnormal{supp}(#1)}}

\newcommand{\Bern}[1]{\ensuremath{\mathcal{B}(#1)}}
\newcommand{\Uni}{\ensuremath{\mathcal{U}}}

\newtheorem{theorem}{Theorem}

\newtheorem{lemma}{Lemma}[section]

\newtheorem{claim}[lemma]{Claim}
\newtheorem{fact}[lemma]{Fact}

\newtheorem{definition}{Definition}

\newtheorem{problem}{Problem}

\newtheorem{rem}[lemma]{Remark}

\newcommand{\Sec}[1]{Section~\ref{#1}}

\newcommand{\Eqn}[1]{Eq.~(\ref{#1})}

\newcommand{\Lem}[1]{Lemma~\ref{#1}}
\newcommand{\Thm}[1]{Theorem~\ref{#1}}

\newcommand{\App}[1]{Appendix~\ref{#1}}
\newcommand{\Def}[1]{Definition~\ref{#1}}

\newcommand{\Fact}[1]{Fact~\ref{#1}}

\newcommand{\Clm}[1]{Claim~\ref{#1}}

\DeclareMathOperator*{\argmax}{argmax}
\DeclareMathOperator*{\argmin}{argmin}

\newcommand{\toShrink}{-.20cm}
\newcommand{\toShrinkEnu}{-.2cm}

\newcommand{\tvd}[2]{\ensuremath{\norm{#1 - #2}_{tvd}}}
\newcommand{\Ot}{\ensuremath{\widetilde{O}}}
\newcommand{\eps}{\ensuremath{\varepsilon}}

\newcommand{\Bracket}[1]{\Big[#1\Big]}
\newcommand{\bracket}[1]{\left[#1\right]}
\newcommand{\paren}[1]{\ensuremath{\left(#1\right)}\xspace}
\newcommand{\card}[1]{\left\vert{#1}\right\vert}

\newcommand{\half}{\ensuremath{{1 \over 2}}}

\newcommand{\IN}{\ensuremath{\mathbb{N}}}

\newcommand{\B}{\ensuremath{\mathcal{B}}}

\newcommand{\norm}[1]{\ensuremath{\|#1\|}}

\newcommand{\poly}{\mbox{\rm poly}}

\newcommand{\ALG}{\ensuremath{\mbox{\sc alg}}\xspace}

\DeclareMathOperator*{\Exp}{\ensuremath{\mathbb{E}}}
\DeclareMathOperator*{\Prob}{\ensuremath{\textnormal{Pr}}}
\renewcommand{\Pr}{\Prob}
\newcommand{\Ex}{\Exp}

\newcommand{\event}{\mathcal{E}}

\newenvironment{tbox}{\begin{tcolorbox}[
		enlarge top by=5pt,
		enlarge bottom by=5pt,
		 breakable,
		 boxsep=0pt,
                  left=4pt,
                  right=4pt,
                  top=10pt,
                  arc=0pt,
                  boxrule=1pt,toprule=1pt,
                  colback=white
                  ]
	}
{\end{tcolorbox}}

\newcommand{\textbox}[2]{
{
\begin{tbox}
\textbf{#1}
{#2}
\end{tbox}
}
}